\documentclass[twoside]{article}

\usepackage[preprint]{aistats2027}
\usepackage[square,numbers]{natbib} 

\usepackage{amsmath}
\usepackage{amssymb}
\usepackage{mathtools}
\usepackage{amsthm}
\usepackage{microtype}
\usepackage{graphicx}
\usepackage{subcaption}
\usepackage{booktabs} 
\usepackage{printlen}
\usepackage{multirow}
\usepackage{tabularx}
\usepackage{bbm}
\usepackage{algorithm}
\usepackage{algorithmic}
\usepackage{tikz}
\usepackage{fancyhdr}
\usepackage{xcolor} 
\usepackage{eso-pic} 
\usepackage{forloop}
\usepackage{caption}
\usepackage{hyperref}
\usepackage[capitalize,noabbrev]{cleveref}
\usepackage{url}
\usepackage{xcolor}
\usepackage{bold-extra} 

\newenvironment{proofsketch}{%
  \proof}{\endproof}

\usetikzlibrary{decorations.pathmorphing,shapes.geometric, arrows.meta, positioning, calc}
\tikzset{snake it/.style={decorate, decoration=snake}}

\newcommand{\subo}{\textsc{SUBo-GFN}}
\newcommand{\subog}{\textsc{SUBo-GFN-RP}}
\newcommand{\subogp}{\textsc{SUBo-GFN-RP+}}

\DeclareRobustCommand{\blackdiamond}{\rotatebox[origin=c]{45}{$\blacksquare$}}
\DeclareRobustCommand{\bigbullet}{\tikz[baseline=-0.6ex]\fill (0,0) circle (0.11cm);}

\theoremstyle{plain}
\newtheorem{theorem}{Theorem}[section]

\newtheorem{lemma}{Lemma}[section]
\newtheorem{corollary}{Corollary}[section]
\theoremstyle{definition}
\newtheorem{definition}{Definition}[section]
\newtheorem{assumption}{Assumption}[section]

\newtheorem*{remark}{Remark}

\begin{document}

%
\runningtitle{Signal from Structure: Exploiting Submodular Upper Bounds in GFNs}

%

\twocolumn[

\aistatstitle{Signal from Structure: Exploiting Submodular Upper Bounds in\protect\\Generative Flow Networks}

\aistatsauthor{Alexandre Larouche \And Audrey Durand}

\aistatsaddress{ Université Laval, Mila \And  Université Laval, Mila, CIFAR AI Chair } ]

\begin{abstract}
Generative Flow Networks (GFlowNets; GFNs) are a class of generative models that learn to sample compositional objects proportional to their a priori unknown value, their reward. We focus on the case where the reward has a specified, actionable structure, namely that it is submodular. We show submodularity can be harnessed to retrieve upper bounds on the reward of compositional objects that have not yet been observed. We provide in-depth analyses of the probability of such bounds occurring, as well as how many unobserved compositional objects can be covered by a bound. Following the Optimism in the Face of Uncertainty principle, we then introduce \subo{}, which uses the submodular upper bounds to train a GFN. We show that \subo{} generates \emph{orders of magnitude} more learning signal than classical GFNs for the same number of queries to the reward function. We demonstrate the effectiveness of \subo{} in terms of distribution matching and high-quality candidate generation in submodular graph tasks.
\end{abstract}

\section{Introduction}

Sampling from hard-to-approximate distributions has been a long-studied problem~\citep{metropolis1953equation}. Some difficulties are aggravated when sampling objects from large combinatorial spaces, e.g. in drug discovery~\citep{shenTacoGFNTargetConditioned2023} and protein design~\citep{jainBiologicalSequenceDesign2022}, namely the generation of the said objects. 
Fortunately, many problems allow for the objects to be composed through a sequence of steps, resulting in compositional objects. In the specific setting where one aims to sample compositional objects proportionally to their unknown \emph{value} (the reward), Generative Flow Networks (GFlowNets, or GFNs)~\citep{bengioFlowNetworkBased2021, bengioGFlowNetFoundations2023} were introduced as a framework to approximate the target sampling distribution through sequential episodic learning.

The success of GFNs is partly driven by their ability to sample a diverse set of candidates in large combinatorial spaces~\citep{zhangLetFlowsTell2023, bengioFlowNetworkBased2021}, paving the way for their application in many domains, such as molecule generation~\citep{koziarskiRGFNSynthesizableMolecular2024a}, sensor selection~\citep{evmorfosGflownetsSensorSelection2023}, and scheduling~\citep{zhangRobustSchedulingGFlowNets2022}. 
A growing body of work studies how 
GFNs can be improved by relying on properties of compositional objects such as symmetry~\citep{maBakingSymmetryGFlowNets2023} and ordering~\citep{chenOrderPreservingGFlowNets2023}, or domain specific knowledge~\citep{koziarskiRGFNSynthesizableMolecular2024a, seoGenerativeFlowsSynthetic2024, mistalCrystalGFNSamplingMaterials2023}. Despite this, few focus on how the reward structure, an elementary component of GFNs, can be exploited~\citep{liu2025efficient, panBetterTrainingGFlowNets2023,jangLearningEnergyDecompositions2023}. 

We consider reward functions that are submodular set functions. Submodularity encodes the \emph{diminishing return} property, where adding an element to a small set must produce an increase in value no less than that of adding the element to one of its superset. This structure naturally occurs in many real-world problems, for example, in genetics~\citep{gasperini2019genome}, sensor selection~\citep{evmorfosGflownetsSensorSelection2023, shamaiahGreedySensorSelection2010} and influence maximization~\citep{domingosMiningNetworkValue2001, zhangSequentialDecisionBased2025}. Submodularity is often assumed in optimization settings, where the objective is to find the input that maximizes a function.  In this paper, we investigate how the submodular structure of the reward function can be leveraged to better estimate the target sampling distribution in GFNs.
\paragraph{Contributions} We demonstrate that from a collection of objects and their rewards, it is possible to infer an upper bound on the reward of unobserved objects.
We provide a theoretical analysis of the probability of obtaining such bounds, as well as how many unobserved objects can be covered by a bound when the collection is obtained by uniform i.i.d. samples. Following the \emph{Optimism in the Face of Uncertainty} (OFU) principle~\cite{brafmanRMAXGeneralPolynomial2002}, we then introduce \subo{}, an approach that uses the submodular upper bounds for training GFNs. 
We find that \subo{} can produce \emph{orders of magnitude} more learning signals per query to the reward function than regular GFNs, which empirically validates our theoretical claims. Furthermore, we theoretically investigate the impact of being optimistic (using upper bounds) on the learned distribution. Finally, we demonstrate the effectiveness of \subo{} in submodular graph tasks, where we discuss its practical implications for both distribution matching and the generation of high-quality candidates.

\section{Related Works}
\label{sec:related_works}

\paragraph{Optimism in the Face of Uncertainty} Our work relies on learning from upper bounds on the reward of unobserved objects (i.e. an optimistic estimate). As such, there is a deep connection with the \emph{Optimism in the Face of Uncertainty} (OFU) principle, which states that when uncertain about the value of an object, one should assume the best case scenario. OFU is commonly used in bandits such as the Upper Confidence Bound (UCB) family of algorithms~\citep{auerFinitetimeAnalysisMultiarmed2002, liContextualbanditApproachPersonalized2010, zhouNeuralContextualBandits2020}. 

\paragraph{Efficient GFNs}
The efficiency in GFNs has been studied from two main angles. One focuses on being efficient w.r.t. queries to the reward function, by optimizing which candidate should be evaluated~\cite{zhu2023sample, kim2024genetic}. The other approach is to leverage reward assumptions, such as differentiability~\citep{liu2025efficient}, to improve the learning signal used by the GFN. Our work exploits how the reward is structured, but the efficiency arises from generating multiple learning signals from gathered observations.


\paragraph{Improving GFN Exploration} Many works have focused on improving the exploration mechanisms in GFNs. Some, such as Generative Augmented Flow Networks~\citep{panGenerativeAugmentedFlow2022} and Sibling-Augmented GFNs~\citep{madanImprovingExplorationSibling2024} implement intrinsic motivation as an exploration mechanism. Local Search GFNs~\citep{kimLocalSearchGFlowNets2023} instead integrate a local search heuristic that backtracks from generated objects to find nearby improvements to candidate solutions. Finally, Thompson Sampling has been studied as a mechanism to guide exploration in GFNs~\citep{rector-brooksThompsonSamplingImproved2023a}. Approaches attempting to improve exploration should be seen as complementary to our work, as they are compatible. However, specifying a reward structure allows us to provide theoretical analyses which pave the way towards problem dependent data-efficiency and exploration mechanisms.

\paragraph{GFNs with Intermediate Learning Signal} 
Traditional GFN settings typically assume that only the value of fully-defined compositional objects can be observed~\cite{bengioFlowNetworkBased2021,bengioGFlowNetFoundations2023}. We consider a relaxed setting where the value of intermediate objects in-the-making can be observed as well. This relaxed setting has been considered in Forward-Looking strategies for GFNs~\citep{panBetterTrainingGFlowNets2023, zhangLetFlowsTell2023, zhangDiffusionGenerativeFlow2023,senderaImprovedDiffusionSampler2024}, which use intermediate values to improve local credit assignment in an attempt to mitigate sparse and delayed learning signals. In many scenarios, this intermediate value can realistically be obtained (e.g., combinatorial optimization~\citep{zhangLetFlowsTell2023}, Bayesian structure learning~\citep{deleuBayesianStructureLearning2022a}, bag generation~\citep{jangLearningEnergyDecompositions2023} and molecule generation~\citep{bengioFlowNetworkBased2021}).
Similarly, LED-GFNs~\citep{jangLearningEnergyDecompositions2023} learn to decompose energy functions dependent on the reward to better assign credit without directly using the intermediate rewards. Crucially, these methods leverage intermediate signals to improve learning in the current trajectory. They do not exploit the properties of the reward to infer information about unobserved objects, as we do. Once again, these approaches are complementary to our work.

\section{Problem Setting}
\label{sec:problem}

We tackle the problem of sampling compositional objects with probability proportional to their value. Let $(\mathcal{S}, \mathcal{A})$ be a tuple describing a directed acyclic graph (DAG), where $\mathcal{S}$ is the state-space and $\mathcal{A}$ is the action-space. A compositional object is built by sequentially selecting elements from $\mathcal A$. A state $s \in \mathcal S$ denotes the set of elements describing an object, and an action $a\in\mathcal A$ corresponds to the addition of element $a$ to $s$. We assume a finite action-space with $|\mathcal{A}|=N$ and denote $\mathcal A(s) := \mathcal A \setminus s$ the elements that are not contained in the state $s$. The action $a\in\mathcal A(s)$ causes a transition from state $s$ to $s' = s \cup \{a\}$.  
The \emph{terminating} state-space $\mathcal{X} := \left\{s \in \mathcal{S} : |s| = C\right\}$ contains all states of fixed target cardinality $C$, with $C \leq N$.

A \emph{trajectory} $\tau:=(s_0,s_1,\dots,s_C)$ corresponds to a sequence of states obtained by sequentially adding elements to the \emph{initial state} $s_0 := \emptyset$ until the cardinality constraint $C$ is met, that is until a terminating state $x\in\mathcal X$ is encountered. We use $\mathcal{T}$ to denote the set of all possible trajectories. We denote the state encountered at position $i \in \{0,1,\dots,C\}$ in a trajectory $\tau$ as $s_i(\tau)$. 


The \emph{reward} function $R:\mathcal{S} \to \mathbb{R}_{\geq0}$ 
assigns a non-negative value to each state. Importantly, $R(x) > 0, \forall x \in \mathcal{X}$. The goal is to sample trajectories terminating in $x \in \mathcal{X}$ with probability $R(x) / Z$, where $Z := \sum_{x \in \mathcal{X}} R(x)$. 
We focus on the setting where the reward function $R$ is a submodular set function.

\begin{assumption}[Submodular reward function]
  \label{ass:submod}
  For any states $s, s^\prime \in \mathcal{S}$ such that $s \subsetneq s^\prime$, and any action $a \in \mathcal{A}(s^\prime)$, $R(s \cup \{a\}) - R(s) \geq R(s^\prime \cup \{a\}) - R(s^\prime)$.
\end{assumption}
Submodularity models diminishing returns (i.e., the added value of an element to a state decreases as the size of the state increases). 
Figure~\ref{fig:gfn_dag} illustrates an example of a DAG with a submodular constraint on the reward function.

\begin{figure}[ht]
  \begin{center}
\begin{tikzpicture}
  \tikzset{vertex/.style = {circle, draw, minimum size=1cm}}
  \tikzset{edge/.style = {->,> = latex}}

  \node[vertex] (s0) at (0, 0) {$s_0 = \emptyset$};
  \node[vertex] (sa) at (1.5, 0.75) {$\{a\}$}; 
  \node[vertex] (sab) at (3, 0.75) {$\{a,b\}$}; 
  \node[vertex] (sabc) at (5, 0) {$\{a,b,c\}$}; 
  \node[vertex] (sb) at (1.5, -0.75) {$\{b\}$}; 
  \node[vertex] (sbc) at (3, -0.75) {$\{b,c\}$}; 

  \draw[edge ] (s0) -- (sa);
  \draw[edge] (sa) -- (sab);
  \draw[edge] (sab) -- (sabc);
  \draw[edge] (s0) -- (sb);
  \draw[edge] (sb) -- (sbc);
  \draw[edge] (sbc) -- (sabc);

  \tikzset{
    ineqbox/.style = {
      draw=blue!70!black,
      fill=blue!5,
      rounded corners,
      thick,
      inner sep=6pt,
      align=left
    },
    ineqtitle/.style = {
      fill=blue!70!black,
      text=white,
      font=\bfseries\small,
      inner xsep=4pt,
      inner ysep=2pt
    }
  }

  \node[ineqbox] (ineq) at (2.5,-2.75) {%
    $R(\{a\}) - R(\emptyset) \ge R(\{a,b,c\}) - R(\{b,c\})$\\[2pt]
    $R(\{b\}) - R(\emptyset) \ge R(\{a,b\}) - R(\{a\})$\\[2pt]
    $R(\{b,c\}) - R(\{b\}) \ge R(\{a,b,c\}) - R(\{a,b\})$
  };

  \node[ineqtitle, anchor=south west] at (ineq.north west) {Submodular constraints};

\end{tikzpicture}
  \end{center}
  \caption{An example DAG with submodular rewards. Submodularity constrains the reward of states, ensuring that there is a diminishing return over time for the reward of later states.} 
  \label{fig:gfn_dag}
 \end{figure}

\subsection{Generative Flow Networks}
\label{sec:gfn}
Generative Flow Networks (GFNs) are a class of methods for sampling compositional objects proportionally to their reward~\citep{bengioFlowNetworkBased2021,bengioGFlowNetFoundations2023}. GFNs learn to distribute the \emph{flow} $F$ induced by the rewards at the terminating states through the DAG such that the incoming and outgoing flows match at every vertex. That is, $F(x) = R(x)$, while the flows at intermediate states $F(s)$ and through edges $F(s'|s)$ are to be determined by the GFN by balancing \emph{flow matching} equations in the DAG. 

The forward and backward policies $P_F(s^\prime|s)$ and $P_B(s|s^\prime)$ are key components of GFNs, modelling a distribution over transitions between states. These policies may be parameterized directly or defined in terms of the flow function $F$ over states and transitions.
A common approach to train a GFN is to start $P_F$ at the initial state $s_0$ and let it generate a sequence of transitions (i.e., a trajectory) until it stops in a terminating state $x \in \mathcal{X}$. Similarly, starting from a terminating state $x \in \mathcal{X}$, $P_B$ can generate a trajectory until it is forced to stop when reaching $s_0$.


 These trajectories are then used to compute a \emph{loss} derived from various flow matching equations~\citep{malkinTrajectoryBalanceImproved2022, deleuBayesianStructureLearning2022a, panBetterTrainingGFlowNets2023} 
 in which the rewards are the \emph{learning signals}. As in several prior works~\citep{panBetterTrainingGFlowNets2023, zhangDiffusionGenerativeFlow2023, zhangLetFlowsTell2023, senderaImprovedDiffusionSampler2024}, we assume rewards at intermediate states in a trajectory can be observed, in addition to rewards at terminating states. 



\section{Submodular Upper Bounds}
\label{sec:bounds}

In this section, we provide theoretical guarantees to characterize the upper bounds that can be obtained given a finite number of trajectories sampled from a DAG with a submodular reward function~(Assumption~\ref{ass:submod}). We provide lower bounds on 1) the probability of obtaining at least one upper bound at any given terminating state; and 2) the expected number of terminating states that are \emph{covered} an upper bound.

\begin{definition}[Submodular upper bounds]
\label{obs:subo}
Let the reward function $R$ be as in Ass.~\ref{ass:submod},
we can upper-bound the reward at terminating state $x \in \mathcal{X}$ given an action $a \in x$ and an intermediate state $s \subsetneq x \setminus \{a\}$:
\begin{equation*}
    \begin{split}
      R(x) &\leq \operatorname{UB}(x| s,a):= R(s\cup \{a\}) - R(s)+ R(x\!\setminus\! \{a\}).
    \end{split}
    \label{eq:bound}
\end{equation*}
\end{definition}
 As the marginal gain of adding element $a$ diminishes as the state $s$ grows, the value of adding $a$ to a small intermediate set $s$ is an overestimate (upper bound) of adding $a$ to the final set $x$ that becomes tighter as $s$ becomes closer to $x$. Therefore an upper bound $\operatorname{UB}(x|s,a)$ may be arbitrarily loose. 
 
%
Computing upper bounds with Definition~\ref{obs:subo} requires a \emph{history} $\mathcal{H}$ of previously visited states $s \in \mathcal{S}$ along with their reward $R(s)$, which is built from the collection of $m$ trajectories sampled by some policy. For our analysis, we will require the following assumptions.
\begin{assumption}[Uniform trajectory collection]
\label{ass:uniform_pf}
The history $\mathcal{H}$ is produced by a policy which collects $m$ i.i.d. trajectories uniformly on $\mathcal{T}$.


\end{assumption}
\begin{assumption}
\label{ass:m_prop}(Minimum parameters)
    We assume that $2 \leq C \leq N$, for the action space size $N$ and the cardinality constraint $C$, and the the number of collected trajectories $m \geq 2$.
\end{assumption}

If a policy only collects a fraction of the trajectories according to Assumption~\ref{ass:uniform_pf}, then we can build the history $\mathcal{H}$ from these trajectories alone, and the following results hold.
Assumption~\ref{ass:m_prop} is mild and only ensures we avoid degenerate cases, or that we do not collect any trajectory at all. We now define the object of study for the remainder of the section.

\begin{definition}[Coverage]
\label{def:coverage}
Fix $m \ge 2$ and let $\mathcal{D}$ be the set of states visited in the $m$ collected trajectories. For $x \in \mathcal{X}$ and $a \in x$, let
$p_{x,a} := x \setminus \{a\}$. A terminating state $x$ is \textit{covered} by a submodular upper bound (Definition~\ref{obs:subo}) when
the following event occurs:
\[
  B_x(m)\! :=\! \bigl\{ \exists\, a \in x,\! \exists\, s \subsetneq p_{x,a}\!:\!
  \{p_{x,a},\, s,\, s \cup \{a\}\} \subseteq \mathcal{D} \bigr\}.
\]
The \emph{coverage} of the terminating state space $\mathcal{X}$ is:
\[
\kappa(m) := \sum_{x \in \mathcal{X}} \mathbbm{1}[B_x(m)],
\]
and the coverage of the unobserved terminating state space is:
\[
\kappa_u(m) := \sum_{x \in \mathcal{X}\setminus\mathcal{D}} \mathbbm{1}[B_x(m)].
\]
\end{definition}

\subsection{Generating Upper Bounds}
\label{sec:distinct_bounds}


We denote $p_{x,a} := x \setminus \{a\}$ the \emph{parent} of terminating state $x$ from which element $a$ is removed. Note that since $|x| = C, \forall x\in\mathcal{X}$, there are $C$ parents per terminating state. Computing an upper bound on $R(x)$ (Definition~\ref{obs:subo}) requires prior observations of rewards in a parent state $p_{x,a}$, some intermediate state $s \subsetneq p_{x,a}$, and \textit{complementary} state $s \cup \{a\}$. Since states $p_{x,a}$ and $s \cup \{a\}$ cannot be encountered in the same trajectory, computing $\operatorname{UB}(x|s,a)$ requires at least two prior trajectories
. From Definition~\ref{def:coverage}, up to three trajectories can be combined to generate an upper bound (one for each required state $p_{x,a}, s \subsetneq p_{x,a}$, and $s\cup \{a\}$). To simplify our analysis, we study the bounds arising strictly from a subset of all possible valid trajectory pairs.
We now establish the criteria for these trajectories.

\begin{definition}[Parent trajectories]
  \label{def:parent_trajs}
  Given a terminating state $x$ and its parent $p_{x,a}$, 
  we denote $\mathcal{T}(p_{x,a})$ the set of trajectories passing through a parent of $x$. 
    \[
    \mathcal{T}(p_{x,a}) = \{\tau \in \mathcal{T}: s_{C-1}(\tau) = p_{x,a}\} \qquad  
    \]
    Note that $|\mathcal{T}(p_{x,a})|$ is identical for all $x\in \mathcal{X}$ and $a\in x$.
\end{definition}
%
\begin{definition}[Compatible trajectories]
  \label{def:a_trajs}
  Given an action $a \in \mathcal{A}$, we denote $\mathcal{T}(a)$ the set of trajectories $\tau \in \mathcal{T}$ such that $s_1(\tau) = \{a\}$:
  \[
  \mathcal{T}(a) = \{\tau \in \mathcal{T}: s_1(\tau) = \{a\}\}.
  \]
   Note that $|\mathcal{T}(a)|$ is identical $\forall a \in \mathcal{A}$.
\end{definition}
%
Figure~\ref{fig:alpha_beta} illustrates an example of trajectories fitting each definition, and how they can be combined. For any terminating state $x \in \mathcal{X}$, if, among $m$ trajectories, some fall in both $\mathcal{T}(p_{x,a})$ and $\mathcal{T}(a)$, then we are guaranteed $B_x(m)$ occurs. Indeed, if $\tau \in \mathcal{T}(a)$, then $s_0(\tau) = \emptyset$ and $s_1(\tau) = \{a\}$ are both in $\mathcal{D}$. If $\tau' \in \mathcal{T}(p_{x,a})$, then $p_{x,a} \in \mathcal{D}$. Since $\emptyset \subsetneq p_{x,a}$, then $B_x(m)$ occurs and a bound on terminating state $x$ can be computed. 



\begin{figure}[ht]
  \begin{center}
      \begin{tikzpicture}
        \tikzset{vertex/.style = {circle, draw, minimum size=1cm}}
    \tikzset{edge/.style = {->,> = latex}}

    \node[vertex, draw=blue, double=red, double distance=0.75pt, thick] (s0) at (0,0) {$s_0 = \emptyset$};
    \node[vertex, red] (sa) at (1.5, 0.75) {$\{a\}$}; 
    \node[vertex, blue, very thick] (sc) at (1.5, -0.75) {$\{c\}$}; 
    \node[vertex, red, very thick] (sab) at (3, 0.75) {$\{a,b\}$}; 
    \node[vertex] (sabc) at (5, 0) {$\{a,b,c\}$}; 
    \node[red] (sdot1) at (5, 1.25) {$\dots$}; 
    \node[blue] (sdot2) at (5, -1.25) {$\dots$}; 
    \node[blue] (sdot3) at (3, -0.75) {$\dots$};

    \draw[edge, red] (s0) -- (sa);
    \draw[edge, blue, very thick] (s0) -- (sc);
    \draw[edge, red] (sa) -- (sab);
    \draw[edge] (sab) -- (sabc);
    \draw[edge,red, very thick] (sab) -- (sdot1);
    \draw[edge] (sa) -- (sdot3);
    \draw[edge, blue] (sc) -- (sdot3);
    \draw[edge, blue, very thick] (sdot3) -- (sdot2);
    \draw[edge] (sdot3) -- (sabc);

  \end{tikzpicture}

  \end{center}
  \caption{An example DAG, showing a pair of trajectories for $\textcolor{red}{\mathcal{T}(p_{x,c})}$ and $\textcolor{blue}{\mathcal{T}(c)}$ for $x = \{a,b,c\}$. This pair of trajectory yields the bound $\operatorname{UB}(\{a,b,c\}| \emptyset, c) = R(\{c\}) - R(\emptyset) + R(\{a,b\})$. The ellipsis ($\dots$) indicates any number of valid transitions are taken. Bold edges and vertices highlight key elements in trajectory.}
  \label{fig:alpha_beta}
 \end{figure}

\subsection{Expected Coverage}
\label{sec:expected_coverage}



We now  provide lower bounds on the coverage (Definition~\ref{def:coverage}), under Assumptions~\ref{ass:uniform_pf} and~\ref{ass:m_prop}. We begin by lower-bounding the probability that the states needed to compute a submodular upper bound (Definition~\ref{obs:subo}) have been visited given $m$ collected trajectories.

  \begin{theorem}[Probability of producing a submodular upper bound]
    \label{thm:lb_proba}
    Given $m$ uniformly collected trajectories (Assumptions~\ref{ass:uniform_pf} and~\ref{ass:m_prop}), for every terminating state $x \in \mathcal{X}$, we have that
    \begin{align*}
        \mathbb{P}(B_x(m))\geq\! C m P_p (1\!-\!P_p)^{m-1}P_m\! -\! \binom{C}{2} m(m\!-\!1)P_p^2.
    \end{align*}
    where $P_p := 1/\binom{N}{C-1}$ is the probability of having a parent trajectory of $x$ and $P_m := 1-(1-1/N)^{m-1}$ is the probability of having at least one compatible trajectory among $m-1$.

    
  \end{theorem}
  \begin{proofsketch}
        We leverage the fact that sampling one single trajectory in $\mathcal{T}(p_{x,\cdot})$ and at least one other in $\mathcal{T}(\cdot)$ is a sufficient condition for event $B_x(m)$ to occur, and then use the second-order Bonferroni inequalities to obtain the lower bound.
  \end{proofsketch}

In Theorem~\ref{thm:lb_proba}, $mP_p(1-P_p)^{m-1}P_m$ lower-bounds the probability that exactly one trajectory passes through a given parent $p_{x,a}$ (occurs w.p. $mP_p(1-P_p)^{m-1}$) while another is a compatible trajectory $a$ (occurs w.p. $P_m$), which covers $x$. Summing over the $C$ parents double counts outcomes where $x$ is covered through several parents. This requires two trajectories through different parents of $x$: we subtract the expected number of such pairs, $\binom{C}{2}m(m-1)P_p^2$. 

    
    
  

  \begin{corollary}[Expected coverage of terminating states]
    \label{cor:lb_coverage}
    Under Assumptions~\ref{ass:uniform_pf} and~\ref{ass:m_prop}, we have
\begin{equation*}
\begin{split}
  \mathbb{E}[\kappa(m)] \ge m(N-C+1)\Bigl(
    &(1-P_p)^{m-1} P_m \\
    &- \tfrac{1}{2}(C-1)(m-1)\,P_p \Bigr),
\end{split}
\end{equation*}
where $P_p := 1/\binom{N}{C-1}$ denotes the probability of having a parent trajectory of $x$ and $P_m := 1-(1-1/N)^{m-1}$ denotes the probability of having a compatible trajectory. The expected coverage of strictly unobserved states is
\[
    \mathbb{E}[\kappa_u(m)] \geq\max\{0, \mathbb{E}[\kappa(m)] - m\}
\]
  \end{corollary}
  \begin{proofsketch}
      The first part follows from the definition of expectation combined with Theorem~\ref{thm:lb_proba}. The second part is implied by the fact that at most $m$ distinct terminating states can be observed with $m$ trajectories. 
  \end{proofsketch}

Corollary~\ref{cor:lb_coverage} can be understood from the trajectory perspective. Each trajectory $\tau$ reaches exactly one parent $q := s_{C-1}(\tau)$, which can bound its $N-C+1$ children in the DAG, giving $m(N-C+1)$ candidate terminating states. A terminating state $q \cup \{a\}$ is then covered if no other trajectory revisits $q$ and a compatible trajectory (Definition~\ref{def:a_trajs}) is sampled, which happens with probability at least $(1-P_p)^{m-1}P_m$. The subtracted term corrects for terminating states shared by the parents of two different trajectories.

\section{\subo{}}
\label{sec:subo}
We introduce \textbf{S}ubmodular \textbf{U}pper-\textbf{Bo}unds GFN (\subo{}), an approach to transform a single query to the reward function into multiple upper bounds on terminating states, given previously sampled trajectories. 
Algorithm~\ref{alg:subo} provides a polynomial time implementation ($O(mC^2(C+N))$) of one training step with \subo{} with an arbitrary flow matching loss $\mathcal{L}$.
We begin by sampling a trajectory $\tau$ ending in $x$ (L1), computing the rewards of the intermediate states, and updating the history $\mathcal H$ (L2). We store the trajectory and its terminating reward in the replay buffer (multiset) (L3). We remove $x$ and its upper bound from the set $\mathcal{U}$ (L4). For efficiency, we maintain a set of successors $\mathrm{Succ}(\cdot)$ that is updated at every first-visit of a state (L6-8). 
We then compute the new triplets of states that can be leveraged to obtain an upper bound (L9-11). These new triplets only leverage the intermediate states of $\tau$ (i.e, not the parent $p=s_{C-1}(\tau)$). Then, we append the set of triplets for which the parent $p$ was the missing link to computing upper bounds, by searching the visited states $\mathcal{D}$ (L12-16). Finally, we compute the new upper bounds from the triplets $\Delta \mathcal{U}$ and their rewards, keeping the tightest(L17-26), and train the GFN (L27-29). See Appendix~\ref{app:complexity} for all details and a discussion on the complexity of Algorithm~\ref{alg:subo}.


\begin{remark}
    Algorithm~\ref{alg:subo} may generate multiple upper bounds for a single terminating state. Keeping only the tightest in $\mathcal{U}$ narrows the gap to the target distribution. 
    Unless specified, $\mathcal{T}_\mathcal{B}$ and $\mathcal{T}_\mathcal{U}$ are sampled i.i.d. and uniformly from their respective sets (L22-23). 
\end{remark}



\begin{algorithm}[t]
\caption{\subo{} (one online step)}
\label{alg:subo}
\begin{algorithmic}[1]
\REQUIRE GFN parameters $\theta$, loss $\mathcal L$, policy $\pi$, replay buffer $\mathcal B$, history $\mathcal H$, visited states $\mathcal D$, upper bounds $\mathcal U(\cdot)$, parents $\mathcal{D}_p$, successor sets $\mathrm{Succ}(\cdot)$
\STATE Sample trajectory $\tau$ ending in $x$ with $\pi$
\STATE $\mathcal{H}(s_i(\tau)) \gets  R(s_i(\tau))$ for $i = 0 \dots C$
\STATE $\mathcal B \gets \mathcal B \cup \{(\tau, R(x))\}$
\STATE $\mathcal{U} \gets \mathcal{U}\setminus \{(x,\mathcal{U}(x))\}$
\STATE $p \gets s_{C-1}(\tau)$; $\Delta\mathcal{D} \gets \{s_i(\tau)\}_{i=0}^{C} \setminus \mathcal{D}$;\! $\mathcal{D} \gets \mathcal{D} \cup \Delta\mathcal{D}$
\FOR{$s \in \Delta\mathcal{D}$, $a \in s$}
  \STATE $\mathrm{Succ}(s \setminus \{a\}) \gets \mathrm{Succ}(s \setminus \{a\}) \cup \{a\}$ 
\ENDFOR
\STATE $\Delta\mathcal E \gets \{(s\setminus\{a\}, a) : s \in \Delta\mathcal{D}, a \in s, s\setminus\{a\} \in \mathcal{D}\}$
\STATE $\Delta\mathcal E \gets \Delta\mathcal E \cup \{(s, a) : s \in \Delta\mathcal{D}, a \in \mathrm{Succ}(s)\}$ 
\STATE $\Delta\mathcal U \! \gets\! \{(q, s, a)\! :\! q \in \mathcal{D}_p, (s,a) \in \Delta\mathcal E, s \subsetneq q, a \notin q\}$ 
\IF{$p \notin \mathcal{D}_p$}
  \STATE $\mathcal{D}_p \gets \mathcal{D}_p \cup \{p\}$
  \STATE $\mathcal{Z} \gets \{s \in \mathcal{D} : s \subsetneq p\}$ 
  \STATE $\Delta\mathcal U \gets \Delta\mathcal U \cup \{(p, s, a) : s \in \mathcal{Z},\, a \in \mathrm{Succ}(s) \setminus p\}$ 
\ENDIF
\FOR{$(q,s,a) \in \Delta\mathcal U$}
\IF{$q \cup \{a\} \not \in \mathcal{D}$}
  \STATE $U \gets \mathcal{H}(s\cup\{a\})\! - \mathcal{H}(s) + \mathcal{H}(q)$ 
  \IF {$q \cup \{a\} \in \mathcal{U}$}
    \STATE $\mathcal{U}(q\cup\{a\}) \gets \min\{\mathcal{U}(q \cup \{a\}, U\}$    
  \ELSE 
      \STATE $\mathcal U(q \cup \{a\}) \gets U$
  \ENDIF
\ENDIF
\ENDFOR
\STATE $\mathcal{T}_\mathcal{B} \leftarrow$ Sample trajectories and rewards from $\mathcal{B}$ 
\STATE $\mathcal{T}_\mathcal{U} \leftarrow$ Sample upper bounds from $\mathcal{U}$ and generate trajectories with $P_{B,\theta}$
\STATE Update parameters $\theta$ with loss $\mathcal L$ on $\mathcal{T}_\mathcal{U} \cup \mathcal{T}_{\mathcal{B}}$. 
\end{algorithmic}
\end{algorithm}

\subsection{Data Efficiency}
%
Recall that learning signals in classical GFNs~\cite{bengioGFlowNetFoundations2023} are typically limited to rewards queried in terminating states, whereas \subo{} can also query rewards in intermediate states. Assuming \subo{} observes a total of $mC$ rewards in $m$ trajectories, as a result of observing intermediate rewards, then a classical GFN could observe up to $mC$ terminating states and their rewards in the same number of queries to the reward function. As Corollary~\ref{cor:lb_coverage} highlights, a key quantity for coverage $\kappa(m)$ is the ratio between the cardinality constraint $C$ and the number of available actions $N$. 

Figure~\ref{fig:ratios} displays the relationship between $C/N$ and the ratio $\kappa(m)/mC$ for different number of trajectories $m$ for \subo{} with a uniform data collection policy. All points above the black dashed line are situations where using \subo{} is worthwhile, as the $mC$ queries to $R$ allow it to cover more than $mC$ terminating states using upper bounds. \textbf{In most scenarios of interest, where $C \ll N$, submodular upper bounds are data-efficient, covering many more terminating states per query to $R$.} This is in stark contrast to the classical GFN training where one query to $R$ gives a learning signal for a single $x \in \mathcal{X}$. 

\begin{figure}[t]
    \centering
    \includegraphics{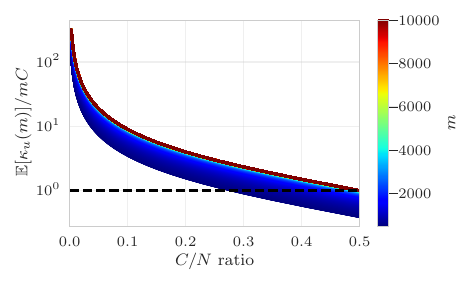}
    \caption{Normalized coverage of unobserved states (Cor.~\ref{cor:lb_coverage}) w.r.t. the cardinality-actions ratio and the number of sampled trajectories~(Assumption~\ref{ass:uniform_pf}) ($m \in [500, 10000]$) for fixed $N=1000$. 
    The dashed line shows the $1:1$ ratio between coverage and query.}
  \label{fig:ratios}
\end{figure}

\subsection{Sampling Distribution Bias}
Let $P$ denote the target sampling distribution, with $P(x) := R(x)/Z$ for terminating state $x\in\mathcal X$, and let $\tilde P$ denote the learned sampling distribution. As \subo{} explores $\mathcal{X}$ entirely, the learning signal becomes $R(x), \forall x \in \mathcal{X}$, as in classical GFNs. Thus, \subo{} converges to $P$ asymptotically, recovering typical guarantees for GFNs~\citep{bengioFlowNetworkBased2021, malkinTrajectoryBalanceImproved2022}.
We now analyze the bias induced in the learned sampling distribution when jointly learning from upper bounds and observed rewards. We study the regime in which all unobserved terminating states have been assigned an upper bound, which is in line with prior GFN convergence guarantees that assume that the reward is observed for all terminating states~\citep{malkinTrajectoryBalanceImproved2022,bengioFlowNetworkBased2021, bengioGFlowNetFoundations2023,deleuBayesianStructureLearning2022a, pmlr-v202-madan23a}. 
%
Let $\mathcal X_{\operatorname{UB}} \subsetneq \mathcal X$ and $\bar{\mathcal X}_{\operatorname{UB}} = \mathcal X \setminus \mathcal X_{\operatorname{UB}}$ respectively denote the set of unobserved terminating states, which are characterized by their upper bound, and the set of observed terminating states, for which the reward is known. Defining the \textit{optimism gap} $\Delta(x):=\operatorname{UB}(x) - R(x)$ at terminating state $x\in\mathcal X_{\operatorname{UB}}$, then jointly learning from rewards on $\bar{\mathcal X}_{\operatorname{UB}}$ and upper bounds on $\mathcal X_{\operatorname{UB}}$ induces the optimistic normalizing constant
\begin{align}
\label{prop:misspec_z}
    \tilde{Z} := \sum_{x \in \mathcal{X}_{\operatorname{UB}}} \operatorname{UB}(x) +\sum_{x \in \bar{\mathcal{X}}_{\operatorname{UB}}} R(x) = Z + \sum_{x \in \mathcal{X}_{\operatorname{UB}}} \Delta(x)
\end{align}
and the resulting biased sampling distribution
\begin{align}
\label{eq:learned_dist}
    \tilde P(x) :=
    \begin{cases}
        \operatorname{UB}(x)/\tilde Z &\quad \text{for } x\in\mathcal X_{\operatorname{UB}}\\
        R(x)/\tilde Z &\quad \text{for } x\in \bar{\mathcal X}_{\operatorname{UB}}.
    \end{cases}
\end{align}
We say that oversampling occurs at terminating state $x \in \mathcal{X}$ when the probability of sampling $x$ under the learned distribution dominates the probability of sampling $x$ under the target distribution, i.e., $\tilde P(x) \geq P(x)$.


\begin{theorem}[Optimism-induced oversampling]
\label{thm:oversampling}
    Learning from upper bounds results in oversampling of the unobserved terminating state $x \in \mathcal{X}_{\operatorname{UB}}$ iff

  \begin{align*}
      \frac{(1-P(x))}{P(x)} \Delta(x) &\geq \sum_{x^\prime \in \mathcal{X}_\mathrm{UB} \setminus \{x\}}\Delta(x').
  \end{align*}

\end{theorem}
\begin{proofsketch}
    This result is obtained by using the definitions of $P(x)$, $\tilde P(x)$, and $\tilde Z$ in $\tilde P(x) \geq P(x)$, and re-arranging terms using simple algebraic manipulations (see Appendix~\ref{app:proof_thm_oversampling} for details).
\end{proofsketch}
Theorem~\ref{thm:oversampling} shows that in order for unobserved terminating states to have an increased sampling probability under the learned distribution $\tilde P$ they need 1) their sampling probability under the target distribution $P$ to be low; and 2) their upper bound to be sufficiently optimistic, i.e., have large optimism gaps. Intuitively, for two terminating states $x, x'\in\mathcal X_{\operatorname{UB}}$ with $\Delta(x) = \Delta(x')$, the terminating state with the lowest sampling probability according to the true distribution will be favoured and vice versa, resulting in the exploration of high uncertainty terminating states.
Naturally, oversampling some terminating states comes at the price of undersampling others. By definition of $\tilde P(x)$ (Equation~\ref{eq:learned_dist}), all observed terminating states are undersampled: $\tilde P(x) < P(x)~\forall x\in\bar{\mathcal X}_{\operatorname{UB}}$. However, our experiments show that this is not a problem in practice, as \subo{} converges more effectively to the target sampling distribution.

\section{Experiments}
\label{sec:experiments}
We now conduct experiment to evaluate the potential of \subo{} on submodular tasks. In Section~\ref{sec:exp_coverage}, we empirically evaluate the coverage of terminating states by \subo{} under two trajectory collection policies: uniform i.i.d. (Assumption~\ref{ass:uniform_pf}), and the GFN's forward policy $P_F$. We then compare the performance of \subo{} against baselines in Section~\ref{sec:exp_learning_dist_graphs}. We investigate primarily graph-based tasks~\citep{zhangLetFlowsTell2023}, inspired by the classical maximum set cover problem, which has a submodular structure. The action space $\mathcal A$ corresponds to adding one vertex from the graph to the current set (state), while the state space $\mathcal S$ corresponds to all the possible sets of vertices under the cardinality constraint $C$. Let $\mathcal N(a)$ denote the neighbouring vertices that can be reached from vertex $a\in\mathcal A$. The reward $R(s) = \bigl|\bigcup_{a \in s} \mathcal N(a)\bigr| / N$ of a given state $s \in \mathcal S$ corresponds to the (normalized) number of unique neighbours that can be reached by following the edges of the vertices in $s$, where $N = |\mathcal A|$. 

For all tasks, we consider cardinality constraints $C \in \{5, 10, 15\}$ and allow a total of 10k queries to $R$. Each experiment is repeated on $10$ random seeds. We report the mean and 95\% confidence intervals (computed from a Student's $t$-distribution). For all experiments, \subo{} trains the GFN with the Trajectory Balance (TB) loss~\citep{malkinTrajectoryBalanceImproved2022}. See Appendix~\ref{app:impl_details} for all implementation details.


\subsection{Terminating state coverage}
\label{sec:exp_coverage}
We investigate the impact of collecting trajectories with the GFN's forward policy $P_F$ instead of the uniform i.i.d. collection policy (Assumption~\ref{ass:uniform_pf}). We consider five random graph models: Erdős–Rényi (ER), Barabási–Albert (BA) and three Configuration Models degree sequence generators (CM-BI, CM-POW, CM-REG). See Appendix~\ref{app:synth_graphs} for details. For each random model, 10 synthetic graphs are generated, one for each random seed. Each graph has $N = 1000$ vertices 
. 


\begin{figure*}
  \centering 
  \begin{subfigure}{0.42\columnwidth}
    \includegraphics{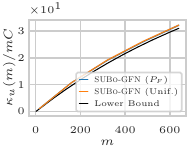}
    \caption{ER}
  \end{subfigure}
   \begin{subfigure}{0.4\columnwidth}
    \includegraphics{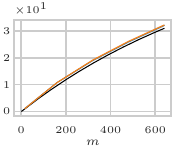}
    \caption{BA}
   \end{subfigure}
   \begin{subfigure}{0.4\columnwidth}
    \includegraphics{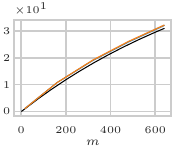}
    \caption{CM-BI}
   \end{subfigure}
   \begin{subfigure}{0.4\columnwidth}
    \includegraphics{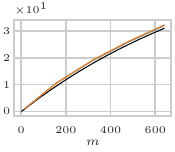}
    \caption{CM-POW}
   \end{subfigure}
   \begin{subfigure}{0.4\columnwidth}
    \includegraphics{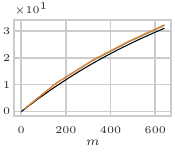}
    \caption{CM-REG}
   \end{subfigure}
  \caption{Normalized coverage of unobserved terminating states for \subo{} with trajectories collected with $P_F$ and the uniform policy (Assumption~\ref{ass:uniform_pf}), against the lower bound (Cor.~\ref{cor:lb_coverage}) for random graphs with $C=15$.}
  \label{fig:coverage_random_1000}
\end{figure*}

Figure~\ref{fig:coverage_random_1000} shows results for cardinality constraint $C=15$ (see Appendix~\ref{app:coverage_add} for $C \in \{5, 10\}$). Using the $P_F$ as the policy $\pi$ in \subo{} is able to cover a comparable amount of terminating states as a policy respecting Assumption~\ref{ass:uniform_pf}. As such, \subo{} is able to cover between $30$ and $171\times$ more terminating states with learning signal than classical GFNs (Appendix~\ref{app:add_results_cov}). Both collection policies result in normalized coverage that is slightly above the theoretical expectation lower bound, providing evidence that the lower bound is fairly tight. 
We hypothesize that getting comparable coverage while using $P_F$ is due to the fact that fewer than $2000$ trajectories are collected during these experiments, which results in at most $125$ gradient steps, when $C = 5$. The forward policy $P_F$ is likely still highly random due to its random initialization. As a result, we conclude that using $P_F$ directly is a viable trajectory collection policy to obtain coverage of the terminating state space, at least during early training.

\subsection{Learning Distributions on Graphs}
\label{sec:exp_learning_dist_graphs}


We conduct experiments to evaluate the potential of \subo{} against Trajectory Balance GFN (TB-GFN)~\cite{bengioFlowNetworkBased2021, malkinTrajectoryBalanceImproved2022} and Forward-Looking GFN (FL-GFN)~\cite{panBetterTrainingGFlowNets2023}. The former is a traditional GFN with TB loss (same as \subo{}), and the latter uses intermediate rewards directly as learning signal with the FL loss. 
We reuse the random graphs from Section~\ref{sec:exp_coverage}.
\begin{figure*}[t]
\centering 
  \begin{subfigure}{0.20\textwidth}
    \includegraphics{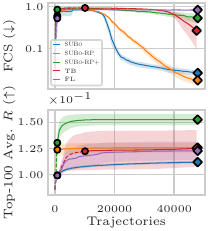}
    \caption{ER}
  \end{subfigure}%
  \begin{subfigure}{0.20\textwidth}
    \includegraphics{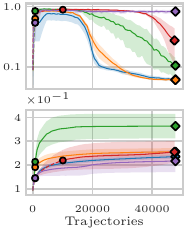}
    \caption{BA}
  \end{subfigure}%
    \begin{subfigure}{0.20\textwidth}
    \includegraphics{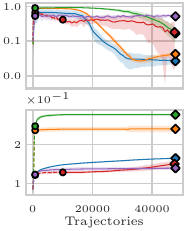}
    \caption{CM-BI}
  \end{subfigure}%
    \begin{subfigure}{0.20\textwidth}
    \includegraphics{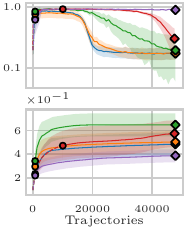}
    \caption{CM-POW}
  \end{subfigure}%
    \begin{subfigure}{0.20\textwidth}
    \includegraphics{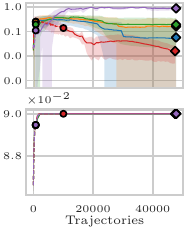}
    \caption{CM-REG}
  \end{subfigure}%
  \caption{Performance metrics of FL/TB-GFNs and \subo{} variants on random graphs with $C=15$. Strategies are trained online (dashed line) until the $\bigbullet$, then trained offline (full line) without further queries to $R$ until the end of the experiment indicated by the $\blackdiamond$.
}
\label{fig:results_graph_learn}
\end{figure*}
All methods learn online until their budget (10k queries) is exhausted, at which point they stop collecting data and learn offline until the end of the experiment.
We measure the Flow Consistency in Subgraph (FCS), introduced as a tractable proxy for the Total Variation (TV) distance between the learned distribution $\tilde P$ and the target distribution $P$~\citep{silvaWhenGFlowNetsLearn2024}. 
We also report the Top-100 Average Reward, a classical application-focused performance measure for GFNs. 

\paragraph{Ablation} We consider three variants of \subo{}. In addition to the vanilla \subo{}, which learns from all available upper bounds, we consider strategies that can prioritize and/or reject upper bounds according to their value. Prioritized approaches sample upper bounds proportionally to their value. Rejection approaches do not learn from upper bounds that are below the maximum observed reward. The \subog{} variant rejects during the online phase and always prioritizes. The \subogp{} variant always rejects and prioritizes. 
These variants allow us to study the impact of choosing which upper bounds to use. All variants of \subo{} use $P_F$ as the data collection policy. 

Figure~\ref{fig:results_graph_learn} shows the results for $C=15$ (see Appendix~\ref{app:learning_add} for $C \in \{5, 10\}$). Overall, vanilla \subo{} generally matches the distribution (lower FCS) much faster than the alternatives. \subog{} comes close, sometimes even having lower FCS than \subo{} despite being more biased towards terminating states with large upper bounds. This allows \subog{} to be competitive both in terms of distribution matching and Top-100 Average Reward. \subogp{} focuses strictly on potentially high-reward terminating states, which improves the Top-100 Avg Reward, at the cost of struggling more to match the distribution. This highlights the versatility of \subo{}.

As CM-REG graphs have approximately identical reward for every terminating state, the target distribution is near-uniform. As a result, every approach except FL-GFNs is comparable in terms of Top-100 Average reward and FCS. We notice that despite FL-GFNs leveraging intermediate rewards, they struggle to be competitive in terms of FCS and Top-100 Average Reward. Two potential causes may explain this: 1) they have much fewer trajectories to train on because they consume as much budget per trajectory as \subo{} and 2) the underlying objective of FL-GFNs is known to have higher bias than TB-GFN~\citep{pmlr-v202-madan23a}.

These results suggest that structured rewards can unlock the potential of GFN in combinatorial domains. This improvement can likely be attributed to the large number of upper bounds that are generated during online training (see Appendix~\ref{app:learning_add}). Additionally, despite the fact that the upper bounds are arbitrarily loose, \subo{} allocates its flow more efficiently than the other baselines, resulting in better distribution matching. This raises questions on how much uncertainty GFNs can tolerate while learning approximately the right distribution in practice.  These results also support that data scale is the most important factor in learning the target, a long-held intuition in the field~\citep{halevyUnreasonableEffectivenessData2009}.

\section{Conclusion}
\label{sec:conclusion}

In this work, we investigate how the structure of the reward function in GFNs can be exploited by focusing on the particular case where the reward is submodular. We first theoretically investigate how submodularity can be harnessed to retrieve upper bounds on the reward of unobserved terminating states.
We provide lower bounds for the coverage of the terminating state space by a submodular upper bound (Theorem~\ref{thm:lb_proba}, Corollary~\ref{cor:lb_coverage}), and show how the coverage varies with the number of sampled trajectories $m$, the number of elements $N$, and the cardinality constraint $C$ of the problem.
Following the OFU principle, we then introduce \subo{} (Algorithm~\ref{alg:subo}), which leverages the submodular upper bounds on unobserved terminating states to train a GFN. \subo{} produces \emph{orders of magnitude} more learning signal than classical GFNs for the same number of queries to the reward function. We further show that the optimistic nature of \subo{} biases the learned distribution towards oversampling unobserved terminating states with high uncertainty or low sampling probability (Theorem~\ref{thm:oversampling}).

We empirically demonstrate that \subo{} can be leveraged to improve distribution matching or high-quality candidate generation. While we focus on graph-based tasks, \subo{} is applicable to \emph{any} problem with a submodular reward (e.g., facility location, sensor selection, feature selection). Our empirical results raise questions, notably as to how much uncertainty GFNs can tolerate while still achieving the desired behavior in practice. 

\paragraph{Limitations and Future Work}
While we provide in-depth analyses when trajectories are collected i.i.d. from a uniform distribution, integrating the GFN interactive learning loop in the generation of upper bounds presents a promising extension. Theoretical results could be further strengthened by considering trajectory splicing, where two sub-trajectories are merged, to produce more compatible trajectories for a parent trajectory. Additionally, the various possible rearrangements of Assumption~\ref{ass:submod} suggest that extending submodular bounds to the full state space could enable scaling beyond fixed depth DAGs. Beyond submodularity, extending our understanding of useful mechanisms for GFNs on well-understood reward structures is an interesting prospect. Lastly, our empirical evaluation is limited to graph, maximum coverage inspired tasks.



\section{AI Use Statement}
In this work, we used generative AI tools for assistance with mathematical proofs, namely learning mathematical tools, providing key elements for proving our claims and providing references to existing works containing key elements. We also used generative AI tools to optimizing code and provide feedback on the correctness of proofs and correctness of methodology. We have not used generative AI tools for the rest of the required disclosure elements. We have reviewed all AI-assisted work, specifically, proofs by checking line-by-line, and code by testing. We take responsibility for the final content of this work, including text, claims, or artifacts produced with the aid of generative AI.


\bibliography{main}
\newpage
\appendix
\onecolumn

\section{Proofs}
\paragraph{Notation} We denote $(n)_k = k!\binom{n}{k} = n(n-1)\dots(n-k+1)$ as the falling factorial and for an event $A$ its complement by $A^c$. In particular, we define $K := C-1$, as this index will be used profusely.
\subsection{Proof of Coverage}
\label{app:proof_thm_coverage}

We define a trajectory $\tau$ as a sequence of states obtained by sequentially adding elements to the \emph{initial state} $s_0 =\emptyset$. The set of all trajectories is denoted $\mathcal{T}$. We denote the state at position $i \in \{0, 1, \dots,C\}$ in a trajectory as $s_i(\tau)$, and $s_0(\tau) = \emptyset, \forall \tau \in \mathcal{T}$.   The following lemma provides the cardinality of $\mathcal{T}$.

\begin{lemma}
\label{lemma:card_T}
    Let $\mathcal{T}$ denote the set of trajectories of length $C$ in the DAG described in Section~\ref{sec:problem}, where the set of actions $\mathcal{A}$ has size $N$. Then,

\[
|\mathcal{T}| = (N)_C
\]
\end{lemma}
\begin{proof}
    Each trajectory is a permutation of $C$ elements from $N$ elements. This yields $\frac{N!}{(N-C)!} = (N)_C$.
\end{proof}

We suppose the uniform trajectory collection policy (Assumption~\ref{ass:uniform_pf}) draws $m$ trajectories, resulting in the sequence $\mathbf{Y}_m:=\left(Y_1, \dots Y_m\right)$, where $Y_n \in \mathcal{T}$ for all $n \in [m]$. Under this policy, each trajectory $\tau \in \mathcal{T}$ has a probability of being sampled that is equal to $\frac{1}{|\mathcal{T}|} = \frac{1}{(N)_C}$, by Lemma~\ref{lemma:card_T} and Assumption~\ref{ass:uniform_pf}.

Denote $\mathcal{D}:=\left\{s_i(Y_n): n \in [m], i \in \{0,1,\dots,C\}\right\}$ as the set of observed states after $m$ trajectories have been sampled. Recall that $p_{x,a}:= x\setminus \{a\}$ denotes a parent of a terminating state $x \in \mathcal{X}$. We recall what it means for a terminating state $x \in \mathcal{X}$ to be covered by an upper bound:

\begin{definition}[Coverage]
\label{def:coverage_app}
Fix $m \ge 1$ and let $\mathcal{D}$ be the set of states visited by the $m$
collected trajectories. For $x \in \mathcal{X}$ and $a \in x$, let
$p_{x,a} := x \setminus \{a\}$. A terminating state $x$ is \emph{covered} if
the event
\[
  B_x(m)\! :=\! \bigl\{ \exists\, a \in x,\! \exists\, s \subsetneq p_{x,a}\!:\!
  \{p_{x,a},\, s,\, s \cup \{a\}\} \subseteq \mathcal{D} \bigr\}
\]
occurs. The \emph{coverage} of $\mathcal{X}$ is
\[
\kappa(m) := \sum_{x \in \mathcal{X}} \mathbbm{1}[B_x(m)].
\]
\end{definition}


There are many ways to cover terminating states according to Definition~\ref{def:coverage_app}. We restrict our analysis to a small fraction of the possible ways to cover a terminating state to simplify the analysis. Define the events $F_{x,a}(\mathcal{D})$ that occurs when parent $p_{x,a}$ and state $\{a\}$ have been encountered in the $m$ trajectories and its union over $a \in x$, $F_{x}$:
\[
F_{x,a}(\mathcal{D}) := \left\{\left\{p_{x,a} \in \mathcal{D}\right\} \cap \left\{\{a\} \in \mathcal{D}\right\}\right\} \quad \text{ and } \quad F_{x}(\mathcal{D}) := \bigcup_{a \in x} F_{x,a}(\mathcal{D}).
\]

This maps to Definitions~\ref{def:parent_trajs} and \ref{def:a_trajs} discussed in Section~\ref{sec:bounds}. Recalling Definition~\ref{def:coverage} for the event that a terminating state $x$ is covered by an upper bound after $m$ trajectories have been collected, $B_x(m)$, we obtain the following lemma.

\begin{lemma}
\label{lemma:fxa_subset}
    Given a collection of $m$ trajectories, $F_{x}(\mathcal{D}) \subseteq B_x(m)$.
\end{lemma}
\begin{proof}
    Recall that $s_0(\tau) = \emptyset, \forall \tau \in \mathcal{T}$ and in particular $\forall Y_n \in \mathbf{Y}_m$ while $m \geq 2$ (Assumption~\ref{ass:m_prop}). If $F_{x,a}(\mathcal{D})$ occurs, then $\{a\} \in \mathcal{D}$ and $p_{x,a} \in \mathcal{D}$. Therefore, the event $B_x(m)$ occurs with state $s = \emptyset \subsetneq p_{x,a}$. This implies $F_{x,a}(\mathcal{D}) \subseteq B_x(m),\forall a \in x$. As each $F_{x,a}(\mathcal{D})$ is a subset of $B_x(m)$, their union remains a subset of $B_x(m)$ and we obtain our result.
\end{proof}


\begin{definition}[Probability of encountering a parent]
\label{def:proba_p}
    Given a parent $p_{x,a}$ for some $x \in \mathcal{X}$ and action $a \in x$, we denote $P_p:= \mathbb{P}(s_K(Y_n) = p_{x,a})$ as the probability of sampling a trajectory passing through the parent.
\end{definition}
\begin{definition}[Probability of encountering single element set]
\label{def:proba_a}
    Given an action $a \in \mathcal{A}$, we denote $P_a:= \mathbb{P}(s_1(Y_n) = \{a\})$ as the probability of a trajectory passing through the state $\{a\}$. Alternatively, this can be understood as the probabiliy of a trajectory starting by selecting action $a$.
\end{definition}

We now establish some properties for these probabilities and their associated events.
\begin{lemma}[Per draw probabilities of state occurrence for $F_{x,a}(\mathcal{D})$]
\label{lemma:per_draw_prob}
For event $F_{x,a}(\mathcal{D})$, for a single trajectory draw $n \in [m]$, we have
\begin{enumerate}
    \item $P_p = \frac{1}{\binom{N}{K}}$ (Probability of a trajectory passing through a parent).
    \item $P_a= \frac{1}{N}$ (Probability of a trajectory choosing one action from $s_0$).
    \item $\{s_K(Y_n) = p_{x,a}\}$ and $\{s_1(Y_n) = \{a\}\}$ are disjoint events for a trajectory $Y_n$.
    \item For any $N$, $P_p \leq \frac{1}{N}$, hence $P_p+P_a \leq \frac{2}{N} \leq 1$ and $1-P_p-P_a \in [0,1]$.
\end{enumerate}
Note that neither $P_p$ nor $P_a$ are dependent on $x, a$, or $n$.
\end{lemma}

\begin{proof}
    \textbf{1)} $s_K(Y_n)$ can be achieved by $K!$ permutations of elements of $p_{x,a}$, followed $N-K$ possibilities to terminate the trajectory by reaching a terminating state. Since the probability of sampling any trajectory is $\frac{1}{(N)_C}$ (Assumption~\ref{ass:uniform_pf}, Lemma~\ref{lemma:card_T}), this gives $\frac{K!(N-K)}{(N)_C} = \frac{K!(N-K) (N-K-1)!}{N!} = \frac{1}{\binom{N}{K}}$.
    
    \textbf{2)} Because of Assumption~\ref{ass:uniform_pf}, then at $s_0$ we have probability $1/N$ of choosing an arbitary action $a$. Alternatively, any $\tau$ s.t. $s_1(\tau) =\{a\}$ is achieved by first choosing action $a$ from the initial state $s_0(\tau)=\emptyset$, and then selecting a permutation of $C-1$ actions from $N-1$ available actions. Since the probability of sampling any trajectory is $\frac{1}{(N)_C}$ (Assumption~\ref{ass:uniform_pf}, Lemma~\ref{lemma:card_T}), this gives $\frac{(N-1)!}{(N-C)!} \frac{(N-C)!}{N!} = \frac{1}{N}$.
    
    \textbf{3)} This follows from the definition of the DAG in Section~\ref{sec:problem}. If $s_1(Y_n) = \{a\}$, then $a \in s_K(Y_n)$ and therefore $s_K(Y_n)$ cannot be $p_{x,a} = x \setminus \{a\}$. Similarly, if $s_K(Y_n) = p_{x,a}$, then $a \not\in p_{x,a}$ and so necessarily $s_1(Y_n) \neq \{a\}$. The events cannot occur on the same trajectory.

    \textbf{4)} $P_p = \frac{1}{\binom{N}{K}}$ is maximal when the binomial coefficient is minimal (when $K$ nears 0, or $K$ nears $N$). Since $2 \leq C \leq N$ and $K = C-1$, then $1 \leq K\leq N-1$. Substituting either value in $P_p$, we obtain $P_p \leq \frac{1}{N}$.
\end{proof}

Recall the sets of parent and compatible trajectories (Definitions~\ref{def:parent_trajs} and~\ref{def:a_trajs}):
\[
\mathcal{T}(p_{x,a}) = \{\tau \in \mathcal{T}: s_K(\tau) = p_{x,a}\} \qquad  \mathcal{T}(a) = \{\tau \in \mathcal{T}: s_1(\tau) = \{a\}\}. 
\]


As such, we may also write $F_{x,a}$ equivalently over trajectories instead of states:
\[
F_{x,a}(\mathcal{D}) = \left\{\exists n \in [m]: Y_n \in \mathcal{T}(p_{x,a})\right\} \cap \left\{\exists n \in [m]: Y_n \in \mathcal{T}(a)\right\}.
\]

We now investigate the probability of this event occurring.
\begin{lemma}[Probability of $F_{x,a}(\mathcal{D})$]
\label{lemma:proba_fxa}
    Let $x \in \mathcal{X}$ denote a terminating state with $a \in x$ and $m$ denote the number of collected trajectories (Assumptions~\ref{ass:uniform_pf},~\ref{ass:m_prop}). Then, the probability of $F_{x,a}(\mathcal{D})$ occurring is
    \[
        \mathbb{P}(F_{x,a}(\mathcal{D})) \geq mP_p\left(1-P_p)^{m-1}(1-(1-P_a)^{m-1}\right)
    \]
    where $P_p$ and $P_a$ are as in Definitions~\ref{def:proba_p},\ref{def:proba_a}. 
    
\end{lemma}

\begin{proof}
    The first part comes from the complement of $F_{x,a}(\mathcal{D})$ followed be the inclusion-exclusion principle.
    \begin{equation}
        \begin{aligned}
            \mathbb{P}(F_{x,a}(\mathcal{D})) &= 1- \mathbb{P}\left(\neg\left(\left\{\exists n \in [m]: Y_n \in \mathcal{T}(p_{x,a})\right\} \cap \left\{\exists n \in [m]: Y_n \in \mathcal{T}(a)\right\}\right)\right) \\ 
            &= 1- \mathbb{P}\left(\left\{\not\exists n \in [m]: Y_n \in \mathcal{T}(p_{x,a})\right\} \cup \left\{\not\exists n \in [m]: Y_n \in \mathcal{T}(a)\right\}\right) \quad & \text{(De Morgan)}\\ 
            &= 1- \mathbb{P}\left(\left\{\not\exists n \in [m]: Y_n \in \mathcal{T}(p_{x,a})\right\}\right) - \mathbb{P}\left(\left\{\not\exists n \in [m]: Y_n \in \mathcal{T}(a)\right\}\right) \quad & \text{(Inclusion-Exclusion)}\\
            &\qquad+  \mathbb{P}\left(\left\{\not\exists n \in [m]: Y_n \in \mathcal{T}(p_{x,a})\right\} \cap \left\{\not\exists n \in [m]: Y_n \in \mathcal{T}(a)\right\}\right) \\ 
            &= 1 - (1-P_p)^m - (1-P_a)^m + (1-P_p-P_a)^m \quad & \text{(Definitions~\ref{def:proba_p},\ref{def:proba_a})}
        \end{aligned}
    \end{equation}

    To simplify our analysis, we define the following event
    \[
    E_n:=\underbrace{\{Y_n \in \mathcal{T}(p_{x,a})\}}_{A_n} \cap \underbrace{\left\{\bigcap_{n' \in [m]\setminus n} Y_{n'} \not\in \mathcal{T}(p_{x,a})\right\}}_{B_n} \cap \underbrace{\left\{\bigcup_{n' \in [m]\setminus n} Y_{n'} \in \mathcal{T}(a)\right\}}_{D_n}.
    \]
    This event is a subset of $F_{x,a}(\mathcal{D})$ where strictly one trajectory, indexed by $n$, passed through the parent $p_{x,a}$. Notice that because only the $n$-th trajectory can satisfy $A_n$, then $E_n \cap E_{n'} = \emptyset$ for $n \neq n'$. Thus, $\mathbb{P}\left(\bigcup_{n=1}^m E_n\right) = \sum_{n=1}^m \mathbb{P}(E_n)$. Now, because the draws are independent by Assumption~\ref{ass:uniform_pf}, $A_n$ is independent from $B_n \cap D_n$, as they relate to entirely different draws (the $n$-th draw has no bearing the the $n'\neq n$ draws). Finally, since $B_n = (B_n \cap D_n) \cup (B_n \cap D_n^c)$, and since $(B_n \cap D_n) \cap (B_n \cap D_n^c) = \emptyset$, we have
    \begin{equation}
        \begin{aligned}
            \mathbb{P}(B_n) &= \mathbb{P}\left((B_n \cap D_n) \cup (B_n \cap D_n^c)\right) \\
            &= \mathbb{P}(B_n \cap D_n) +\mathbb{P}(B_n \cap D_n^c) - \mathbb{P}\left((B_n \cap D_n) \cap (B_n \cap D_n^c)\right) & \quad \text{(Inclusion-exclusion)}\\
            &= \mathbb{P}(B_n \cap D_n) +\mathbb{P}(B_n \cap D_n^c) & \quad \text{(Disjoint sets)}\\
            \mathbb{P}(B_n \cap D_n) &= \mathbb{P}(B_n) -\mathbb{P}(B_n \cap D_n^c). \\
        \end{aligned}
    \end{equation}
    Combining these facts, we obtain
    \begin{equation}
    \begin{aligned}
        \mathbb{P}(E_n)&=\mathbb{P}(A_n)\mathbb{P}(B_n \cap D_n) \quad & \text{(Independence)} \\
        &=\mathbb{P}(A_n)(\mathbb{P}(B_n) - \mathbb{P}(B_n \cap  D_n^c)). \\
    \end{aligned}
    \end{equation}
    By Definitions~\ref{def:proba_p},\ref{def:proba_a}, we have
    \begin{itemize}
        \item $\mathbb{P}(A_n) = P_p$, the probability of drawing a trajectory that passes through a parent.
        \item $\mathbb{P}(B_n) = (1-P_p)^{m-1}$, the probability of never drawing a trajectory passing through a parent in $m-1$ draws.
        \item $\mathbb{P}(B_n \cap D_n^c) = (1-P_p-P_a)^{m-1}$, the probability of never drawing a trajectory passing through a parent $p_{x,a}$ or its associated action $a$ in $m-1$ draws.
    \end{itemize}

    Thus,
    \begin{equation}
        \begin{aligned}
                \mathbb{P}\left(F_{x,a}(\mathcal{D})\right) \geq \mathbb{P}\left(\bigcup_{n=1}^m E_n\right) = \sum_{n=1}^m \mathbb{P}(E_n) 
                &= m P_p\left((1-P_p)^{m-1} -(1-P_p-P_a)^{m-1}\right)\\
                &\geq m P_p\left((1-P_p)^{m-1} -(1-P_p-P_a + P_p P_a)^{m-1}\right)\\  
                &=m P_p\left((1-P_p)^{m-1} -(1-P_p)^{m-1}(1-P_a)^{m-1}\right)\\
                &=mP_p\left(1-P_p)^{m-1}(1-(1-P_a)^{m-1}\right).
        \end{aligned}
    \end{equation}
    The second line is obtained by noticing that $P_p P_a \geq 0$, since they are probabilities. This introduces looseness, but this expression is easier to work with.
\end{proof}

We now upper bound the probability of $F_{x,a}(\mathcal{D}) \cap F_{x,b}(\mathcal{D})$ occurring for some actions $a,b \in x$, $a \neq b$.

\begin{lemma}
\label{lemma:proba_intersection}
    Let $x \in \mathcal{X}$ denote a terminating state, $a,b$ denote some elements of $x$ such that $a\neq b$, and $m$ be the number of collected trajectories (Assumptions~\ref{ass:uniform_pf},~\ref{ass:m_prop}). Then,
    \[
     \mathbb{P}\left(F_{x,a}(\mathcal{D}) \cap F_{x,b}(\mathcal{D})\right) \leq m(m-1)P_p^2 .
    \]
\end{lemma}

\begin{proof}
Define event $G_a := \left\{p_{x,a} \in \mathcal{D}\right\}$. Clearly, $F_{x,a}(\mathcal{D})\subseteq G_a$, since $G_a$ is a requirement for $F_{x,a}(\mathcal{D})$ to occur but not the other way around. Thus,

\begin{equation}
\begin{aligned}
    \mathbb{P}\left(F_{x,a}(\mathcal{D}) \cap F_{x,b}(\mathcal{D})\right) &\leq \mathbb{P}\left(G_a \cap G_b\right) \\
    &= 1 - \mathbb{P}\left(G_a^c \cup G_b^c\right) \quad & \text{(De Morgan)}\\
    &= 1 - \mathbb{P}\left(G_a^c\right) -  \mathbb{P}\left(G_b^c\right) + \mathbb{P}\left(G_a^c \cap G_b^c\right) \quad &\text{(Inclusion-exclusion)} \\
    &= 1 - 2(1-P_p)^m + (1-2P_p)^m \quad & \text{(Definitions~\ref{def:proba_p},\ref{def:proba_a})} \\
\end{aligned}
\end{equation}
Define $g(y) = (1-y)^m$ and $f(y) = g(y + P_p) - g(y)$ then
\begin{equation*}
    \begin{aligned}
        1 - 2(1-P_p)^m + (1-2P_p)^m  &= g(0) - 2g(P_p) + g(2P_p) \\
        &= f(P_p) - f(0) \\
        &= \int_{0}^{P_p} f'(w)dw \quad & \text{(Fundamental Thm of Calculus)}\\
        &= \int_{0}^{P_p} g'(w + P_p) - g'(w) dw \\
        &= \int_{0}^{P_p} \int_{w}^{w+P_p}g''(z) dzdw & \text{(Fundamental Thm of Calculus)} \\
        &= \int_{0}^{P_p} \int_{0}^{P_p}g''(h+w) dhdw. & \text{(Substitution rule)}\\
    \end{aligned}
\end{equation*}
Now, $g''(y) = m(m-1)(1-y)^{m-2}$. Note that the integrand's input goes from $0$ to $2P_p$, and that $N \geq 2$, $P_p \leq \frac{1}{N} \leq \frac{1}{2}$ (Lemma~\ref{lemma:per_draw_prob}), and so the input is in the range $[0,1]$. Then, $g''(h+w) \leq m(m-1)(1 - 0)^{m-2}$, and we get 
\[
\int_{0}^{P_p} \int_{0}^{P_p}g''(h+w) dhdw \leq m(m-1) \int_{0}^{P_p} \int_{0}^{P_p}1 = m(m-1)P_p^2.
\]
Thus, 
\[
\mathbb{P}\left(F_{x,a}(\mathcal{D}) \cap F_{x,b}(\mathcal{D})\right)  \leq m(m-1)P_p^2. 
\]
\end{proof}

We are now ready to state the main results.
\begin{theorem}[Probability that a terminating state is covered]
    Let $N$ denote the number of actions $|\mathcal{A}|$ and $C$ denote the cardinality constraint and let $B_x(m)$ denote the event that $x$ is covered by a submodular upper bound for any terminating state $x \in \mathcal{X}$. Then, under Assumptions~\ref{ass:uniform_pf} and \ref{ass:m_prop}, we have
    \begin{equation*}
        \begin{aligned}
        \mathbb{P}(B_x(m)) &\geq C m P_p (1-P_p)^{m-1}\left(1 - (1-P_a)^{m-1}\right)  - \binom{C}{2} m(m-1)P_p^2.
        \end{aligned}
    \end{equation*}
\end{theorem}

\begin{proof}
    Recall Lemma~\ref{lemma:fxa_subset}. Then $\mathbb{P}(B_x(m)) \geq \mathbb{P}(F_x(m))$. Bonferroni's inequalities give us that 
\begin{equation*}
    \begin{aligned}
        \mathbb{P}(F_{x}(\mathcal{D})) &= \mathbb{P}\left(\bigcup_{a \in x} F_{x,a}(\mathcal{D})\right)\\  
         &\geq \sum_{a \in x} \mathbb{P}\left(F_{x,a}(\mathcal{D})\right) - \frac{1}{2}\sum_{a,b \in x: a \neq b} \mathbb{P}\left(F_{x,a}(\mathcal{D}) \cap F_{x,b}(\mathcal{D})\right)\\  
         &\geq C m P_p (1-P_p)^{m-1}\left(1 - (1-P_a)^{m-1}\right)  - \binom{C}{2} m(m-1)P_p^2  \quad & \text{(Lemmas~\ref{lemma:proba_fxa}, \ref{lemma:proba_intersection})}\\  
    \end{aligned}
\end{equation*}
\end{proof}

\begin{corollary}
Under Assumptions~\ref{ass:uniform_pf} and~\ref{ass:m_prop}, we have
    \begin{equation*}
        \begin{aligned}
            \mathbb{E}[\kappa(m)] &\geq m(N-C+1) \left(\left(1-P_p\right)^{m-1}\left(1 - \left(1-P_a\right)^{m-1}\right)  - \frac{(C-1)(m-1)P_p}{2}\right)\\
        \end{aligned}
    \end{equation*}
\end{corollary}
\begin{proof}
    We apply the definition of expected value.
    \begin{equation*}
        \begin{aligned}
            \mathbb{E}[\kappa(m)] &= \mathbb{E}\left[\sum_{x \in \mathcal{X}} \mathbbm{1}[B_x(m)]\right] \\
            &= \sum_{x \in \mathcal{X}}\mathbb{E}\left[ \mathbbm{1}[B_x(m)]\right] \\
            &= \sum_{x \in \mathcal{X}} \mathbb{P}\left(B_x(m)\right) \\
            &\geq \binom{N}{C} \left(C m P_p (1-P_p)^{m-1}\left(1 - (1-P_a)^{m-1}\right)  - \binom{C}{2} m(m-1)P_p^2\right)\\
            &= m(N-C+1) \left(\left(1-P_p\right)^{m-1}\left(1 - \left(1-P_a\right)^{m-1}\right)  - \frac{(C-1)(m-1)P_p}{2}\right) & \quad \text{(Algebra)}\\
        \end{aligned}
    \end{equation*}
\end{proof}

As such, this theorem demonstrates that as the ratio $N/C$ grows, a larger number of submodular upper bounds will be obtained in expectation. For the coverage for strictly unobserved terminating states after $m$ trajectories have been collected, we have the following corollary.

\begin{corollary}
Under Assumptions~\ref{ass:uniform_pf} and~\ref{ass:m_prop}, let $\kappa_u(m)$ denote the number of covered terminating states that have not been observed when sampling the $m$ trajectories. Then,
    \begin{equation*}
        \begin{aligned}
            \mathbb{E}[\kappa_u(m)] \geq \max\{0, \mathbb{E}[\kappa(m)] - m\}
        \end{aligned}
    \end{equation*}
\end{corollary}
\begin{proof}
    This follows from the fact that after $m$ collected trajectories, at most $m$ unique terminating states have been observed. Therefore $\kappa_u(m) \geq \kappa(m) - m$, and $\kappa_u(m) \geq 0, \forall m \geq 0$. Taking the expectation, we get the result. 
\end{proof}

\subsection{Proof of Theorem~\ref{thm:oversampling}}
\label{app:proof_thm_oversampling}
\begin{proof}
    Let $\bar{\mathcal{X}}_\mathrm{UB}\subseteq \mathcal{X} $  be the set of observed terminating states (i.e. for which we know the reward) and let $\mathcal{X}_{\operatorname{UB}} := \mathcal{X} \setminus \bar{\mathcal{X}}_\mathrm{UB}$ denote the subset of unobserved terminating states. We assume we have an upper bound for every $x \in \mathcal{X}_\mathrm{UB}$. Then, given $x \in \mathcal{X}_{\operatorname{UB}}$, we have

\begin{equation}
    \begin{split}
  \tilde P (x) &\geq P(x) \\ 
  \frac{\operatorname{UB}(x)}{\tilde Z} &\geq \frac{R(x)}{Z} \\ 
   \operatorname{UB}(x) &\geq \frac{R(x)\tilde Z}{Z} \\ 
   &= R(x)\frac{Z + \sum_{x^\prime \in \mathcal{X}_{\operatorname{UB}}}\Delta(x^\prime)}{Z} \\ 
   &= R(x)\left(1+\frac{\sum_{x^\prime \in \mathcal{X}_\mathrm{UB}}\Delta(x^\prime)}{Z} \right) \\ 
  &= R(x) + R(x) \frac{\sum_{x^\prime \in \mathcal{X}_\mathrm{UB}}\Delta(x^\prime)}{Z} \\
  &= R(x) + P(x) \sum_{x^\prime \in \mathcal{X}_\mathrm{UB}}\Delta(x^\prime) \\
    \end{split}
\end{equation}
which implies
\begin{equation}
    \begin{split}
              \Delta(x) &\geq P(x) \sum_{x^\prime \in \mathcal{X}_\mathrm{UB}}\Delta(x^\prime)\\
  &= P(x) \Delta(x) + P(x) \sum_{x^\prime \in \mathcal{X}_\mathrm{UB} \setminus \{x\}}\Delta(x^\prime)\\
  (1-P(x))\Delta(x) &\geq P(x) \sum_{x^\prime \in \mathcal{X}_\mathrm{UB} \setminus \{x\}}\Delta(x^\prime)\\
  \frac{(1-P(x))}{P(x)}\Delta(x) &\geq \sum_{x^\prime \in \mathcal{X}_\mathrm{UB} \setminus \{x\}}\Delta(x^\prime)\\
    \end{split}
\end{equation}
Since $R(x) > 0$ for all $x \in \mathcal{X}$, we have $Z > 0$ and $P(x) > 0$, and $\tilde Z \geq Z > 0$ as $\Delta(x) \geq 0$. Each step only adds the same term to both sides or multiplies or divides both sides by one of these strictly positive quantities. Such steps are reversible, which yields the double implication, and preserve the direction of the inequality.
\end{proof}

\section{Details on Random Graphs Models}
\label{app:synth_graphs}
For random graph instances using random Barabási–Albert (BA), Erdős–Rényi (ER) and Configuration Models (CM) with various degree sequences, we generate graphs with $N = 1000$ vertices. BA-graphs use preferential attachment as a mechanism to distribute the edges, whereas ER-graphs are more random, where each edge has the same probability of occurrence. BA-graphs are parameterized by an integer which determines how many neighbors a new node added to the graph should have. The probability of a node being a neighbor is proportional to the degree of this node. We choose $3$ as our parameter. For ER-graphs, we choose parameter value $0.005$, which is the probability that an edge between two vertices exists. This probability is constant for the entire ER-graph. We also consider Configuration Models, which allow much more flexibility by directly generating the degree for every node. We consider three degree sequence generation methods. The first is a power law sequence, where the probability that a vertex has degree $d$ is $\propto d^{-2.5}$. Though we chose $d \in [2, 10000]$ while generating the graph, we confirmed that no degree sequence had a degree greater than $N-1=999$ (the total probability of $d \geq 1000$ in this model is $5.98 \times 10^{-5}$). The second method generates a regular degree sequence, where every node has $d=5$ edges. The last method generates a bimodal distribution of degrees, where the vertices are split into two groups, one containing 10\% of the vertices and with degree $d = 25$ and the other containing 90\% of the vertices and with degree $d = 3$. We name the graphs generated from these methods CM-POW, CM-REG and CM-BI respectively. 

Each type of graph models a different edge distribution (and thus reward distribution). BA Graphs and CM-POW generate scale-free networks which introduce a small set of vertices that are extremely well connected, i.e. hubs. This results in a very sparse and peaky reward distribution. CM-BI also has hubs, but their connectivity is much more moderate and fixed to degree $k = 25$. High rewards require targeting groups of vertices who are part of the 10\% of moderately well connected vertices. ER graphs and CM-REG graphs emulate something akin to uniform distribution of degrees across vertices. While the degree of a vertex in ER graphs is probabilistic, it is deterministic in CM-REG graphs, as the degree sequence enforces $d = 5$.




\section{Additional results}
\label{app:add_results}
We provide here additional results for the terminating state coverage experiment as well as the learning distributions on graphs experiments.

\subsection{Additional results for Terminating States Coverage}
\label{app:add_results_cov}
We provide additional results on the comparison between collection policies for \subo{}, namely the GFN's forward policy $P_F$ and the policy collecting trajectories i.i.d. and uniformly on the trajectory space (Assumption~\ref{ass:uniform_pf}.
The setup is identical to that presented in the main body of the paper. Figures~\ref{fig:cov_er},~\ref{fig:cov_ba},~\ref{fig:cov_cmbi},~\ref{fig:cov_cmpow},~\ref{fig:cov_cmreg} display the results.

\label{app:coverage_add}
\begin{figure}[H]
  \centering 
  \begin{subfigure}{0.32\columnwidth}
    \includegraphics{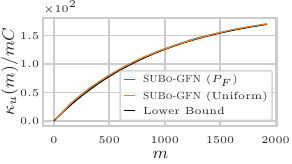}
    \caption{$C = 5$}
  \end{subfigure}
  \begin{subfigure}{0.32\columnwidth}
    \includegraphics{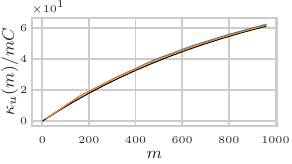}
    \caption{$C = 10$}
  \end{subfigure}
  \begin{subfigure}{0.32\columnwidth}
    \includegraphics{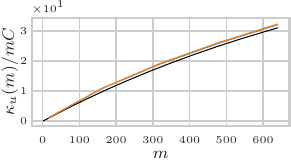}
    \caption{$C = 15$}
  \end{subfigure}
\caption{Normalized coverage of unobserved terminating states for \subo{} with trajectories collected with $P_F$ and the uniform policy (Assumption~\ref{ass:uniform_pf}), against the lower bound (Cor.~\ref{cor:lb_coverage}) for ER graphs.}
  \label{fig:cov_er}
\end{figure}

\begin{figure}[H]
  \centering 
  \begin{subfigure}{0.32\columnwidth}
    \includegraphics{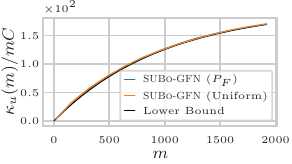}
    \caption{$C = 5$}
  \end{subfigure}
  \begin{subfigure}{0.32\columnwidth}
    \includegraphics{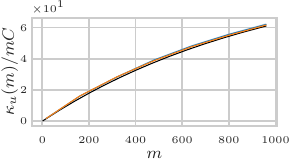}
    \caption{$C = 10$}
  \end{subfigure}
  \begin{subfigure}{0.32\columnwidth}
    \includegraphics{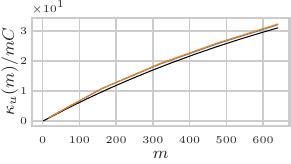}
    \caption{$C = 15$}
  \end{subfigure}
\caption{Normalized coverage of unobserved terminating states for \subo{} with trajectories collected with $P_F$ and the uniform policy (Assumption~\ref{ass:uniform_pf}), against the lower bound (Cor.~\ref{cor:lb_coverage}) for BA graphs.}
  \label{fig:cov_ba}
\end{figure}

\begin{figure}[H]
  \centering 
  \begin{subfigure}{0.32\columnwidth}
    \includegraphics{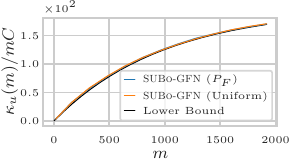}
    \caption{$C = 5$}
  \end{subfigure}
  \begin{subfigure}{0.32\columnwidth}
    \includegraphics{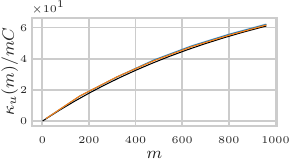}
    \caption{$C = 10$}
  \end{subfigure}
  \begin{subfigure}{0.32\columnwidth}
    \includegraphics{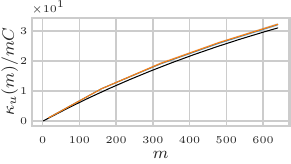}
    \caption{$C = 15$}
  \end{subfigure}
\caption{Normalized coverage of unobserved terminating states for \subo{} with trajectories collected with $P_F$ and the uniform policy (Assumption~\ref{ass:uniform_pf}), against the lower bound (Cor.~\ref{cor:lb_coverage}) for CM-BI graphs.}
  \label{fig:cov_cmbi}
\end{figure}

\begin{figure}[H]
  \centering 
  \begin{subfigure}{0.32\columnwidth}
    \includegraphics{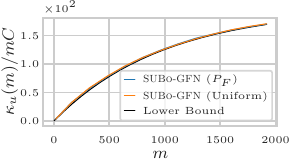}
    \caption{$C = 5$}
  \end{subfigure}
  \begin{subfigure}{0.32\columnwidth}
    \includegraphics{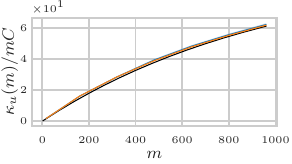}
    \caption{$C = 10$}
  \end{subfigure}
  \begin{subfigure}{0.32\columnwidth}
    \includegraphics{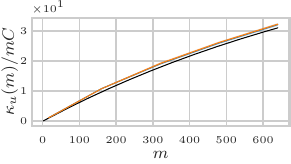}
    \caption{$C = 15$}
  \end{subfigure}
\caption{Normalized coverage of unobserved terminating states for \subo{} with trajectories collected with $P_F$ and the uniform policy (Assumption~\ref{ass:uniform_pf}), against the lower bound (Cor.~\ref{cor:lb_coverage}) for CM-POW graphs.}
  \label{fig:cov_cmpow}
\end{figure}

\begin{figure}[H]
  \centering 
  \begin{subfigure}{0.32\columnwidth}
    \includegraphics{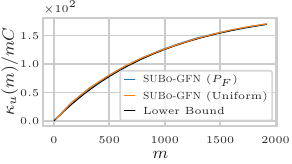}
    \caption{$C = 5$}
  \end{subfigure}
  \begin{subfigure}{0.32\columnwidth}
    \includegraphics{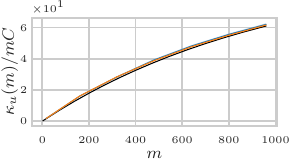}
    \caption{$C = 10$}
  \end{subfigure}
  \begin{subfigure}{0.32\columnwidth}
    \includegraphics{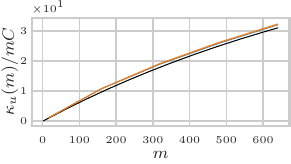}
    \caption{$C = 15$}
  \end{subfigure}
\caption{Normalized coverage of unobserved terminating states for \subo{} with trajectories collected with $P_F$ and the uniform policy (Assumption~\ref{ass:uniform_pf}), against the lower bound (Cor.~\ref{cor:lb_coverage}) for CM-REG graphs.}
  \label{fig:cov_cmreg}
\end{figure}

The conclusion is straightforward. Using the $P_F$ as the policy $\pi$ in \subo{} is able to cover a comparable amount of terminating states as a policy respecting Assumption~\ref{ass:uniform_pf}. Both methods are always very slightly above the theoretically expectation lower bound, providing evidence that that the lower bound is fairly tight, despite the fact that it considers only a subset of all trajectories that can be used to compute submodular upper bounds. Furthermore, the difference in normalized coverage scale also fits with Figure~\ref{fig:ratios}, as $C=5$ has a smaller $C/N$ ratio and thus is expected to produce a larger normalized coverage that $C \in \{10, 15\}$.


\section{Additional Results for Learning Distributions on Graphs}
We provide additional results on distribution learning on graph tasks with submodular rewards. We follow the same experimental setup as that presented in the main body of the paper.
\label{app:learning_add}
\subsection{Random Graphs}
Figures~\ref{fig:ba},\ref{fig:er}, \ref{fig:cm_bi}, \ref{fig:cm_pow}, and \ref{fig:cm_reg} show the full results with $C \in \{5, 10, 15\}$ for the random graphs, including performance metrics (FCS, Loss curves, Top-100 Average Rewards) as well as the terminating state coverage for the upper bounds generated by \subo{}-(RP+).

\begin{figure}[H]
  \centering 
  \begin{subfigure}{0.32\textwidth}
    \includegraphics{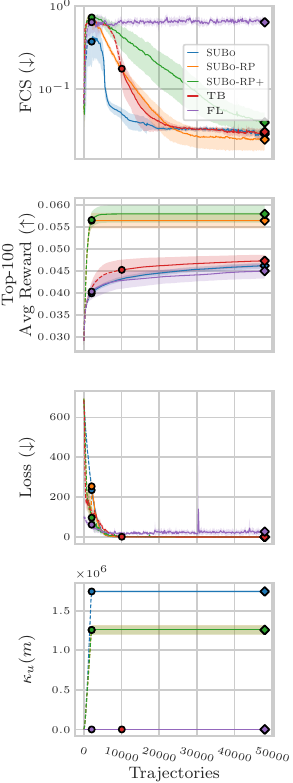}
    \caption{$C=5$}
  \end{subfigure}
    \begin{subfigure}{0.32\textwidth}
    \includegraphics{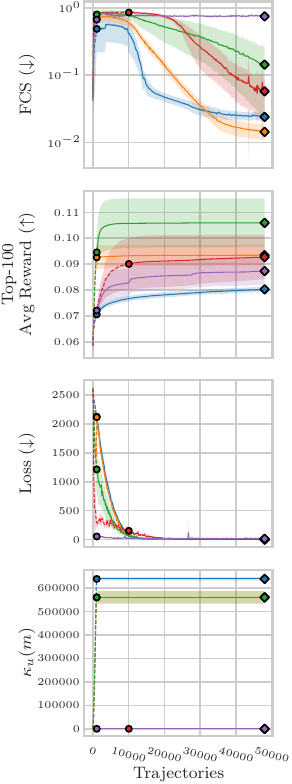}
    \caption{$C=10$}
  \end{subfigure}
    \begin{subfigure}{0.32\textwidth}
    \includegraphics{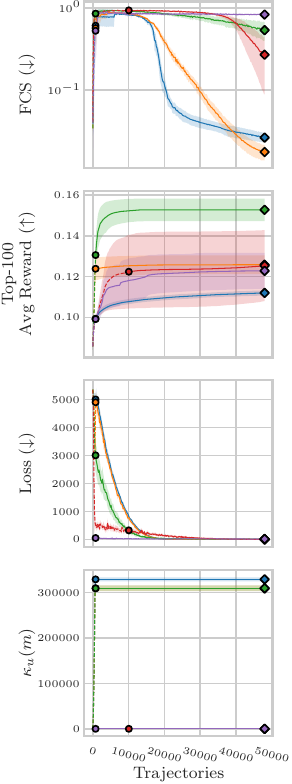}
    \caption{$C=15$}
  \end{subfigure}
  \caption{FCS, Top-100 Average Rewards, Loss, and Coverage for \subo{} and FL/TB-GFNs on ER graphs.} \label{fig:er}
\end{figure}

\begin{figure}[H]
  \centering 
  \begin{subfigure}{0.32\textwidth}
    \includegraphics{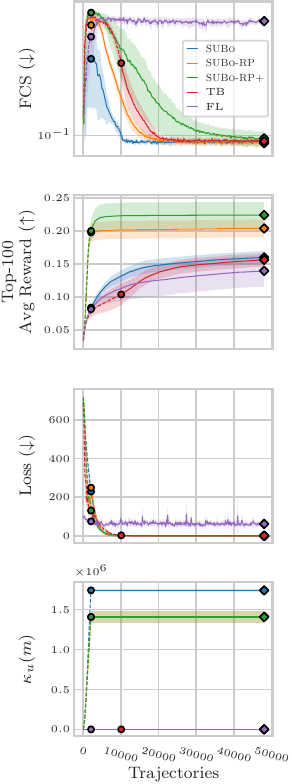}
    \caption{$C=5$}
  \end{subfigure}
    \begin{subfigure}{0.32\textwidth}
    \includegraphics{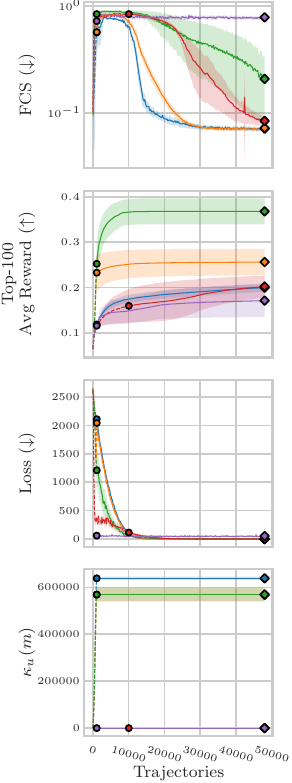}
    \caption{$C=10$}
  \end{subfigure}
    \begin{subfigure}{0.32\textwidth}
    \includegraphics{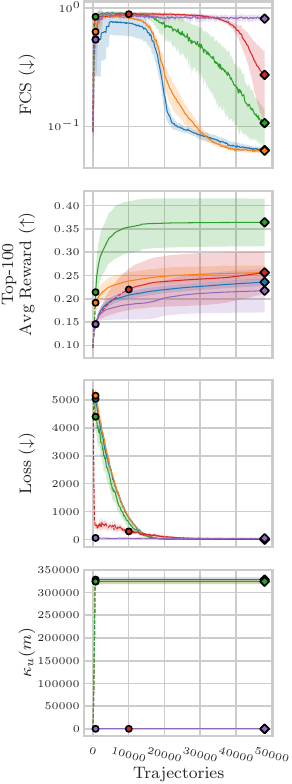}
    \caption{$C=15$}
  \end{subfigure}
  \caption{FCS, Top-100 Average Rewards, Loss, and Coverage for \subo{} and FL/TB-GFNs on BA graphs.} \label{fig:ba}
\end{figure}

\begin{figure}[H]
  \centering 
  \begin{subfigure}{0.32\textwidth}
    \includegraphics{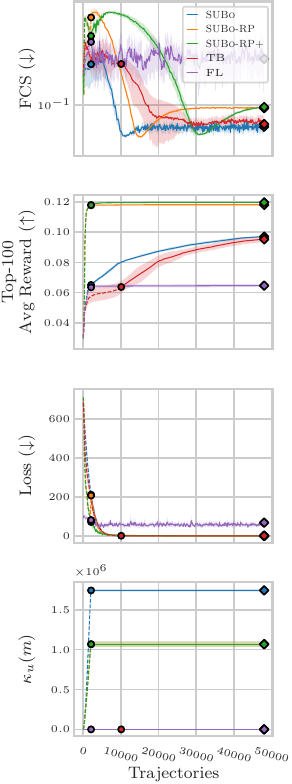}
    \caption{$C=5$}
  \end{subfigure}
    \begin{subfigure}{0.32\textwidth}
    \includegraphics{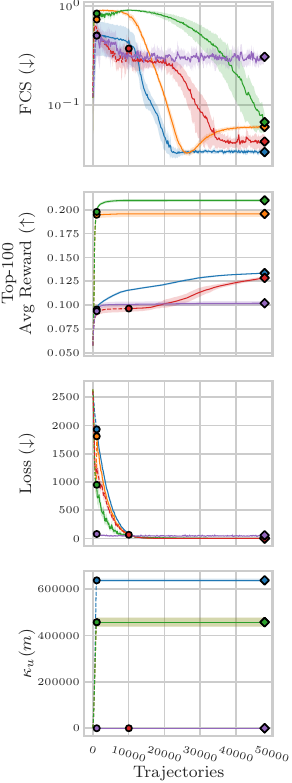}
    \caption{$C=10$}
  \end{subfigure}
    \begin{subfigure}{0.32\textwidth}
    \includegraphics{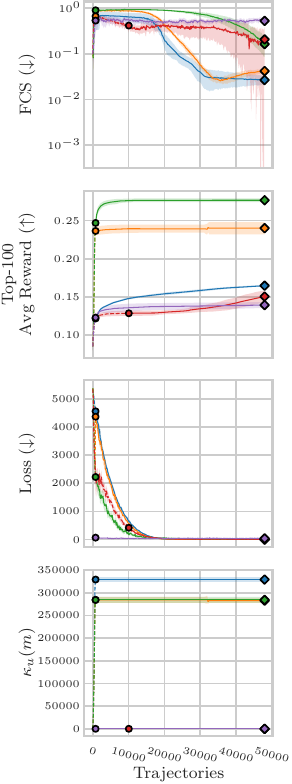}
    \caption{$C=15$}
  \end{subfigure}
  \caption{FCS, Top-100 Average Rewards, Loss, and Coverage for \subo{} and FL/TB-GFNs on CM-BI graphs.} \label{fig:cm_bi}
\end{figure}

\begin{figure}[H]
  \centering 
  \begin{subfigure}{0.32\textwidth}
    \includegraphics{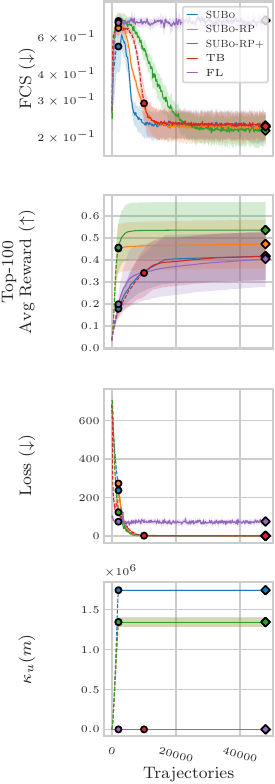}
    \caption{$C=5$}
  \end{subfigure}
    \begin{subfigure}{0.32\textwidth}
    \includegraphics{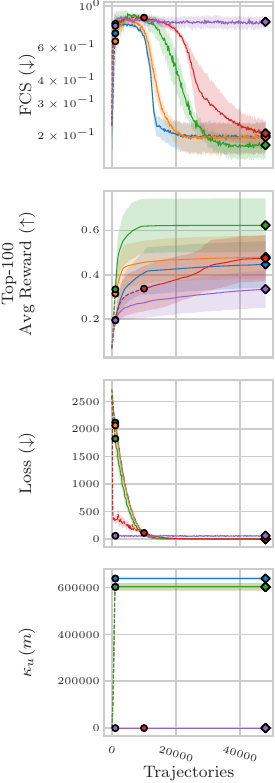}
    \caption{$C=10$}
  \end{subfigure}
    \begin{subfigure}{0.32\textwidth}
    \includegraphics{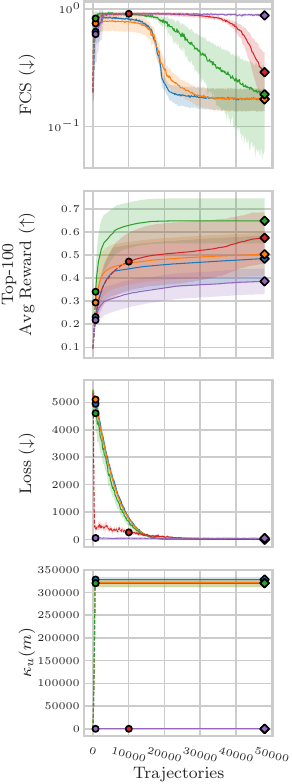}
    \caption{$C=15$}
  \end{subfigure}
  \caption{FCS, Top-100 Average Rewards, Loss, and Coverage for \subo{} and FL/TB-GFNs on CM-POW graphs.} \label{fig:cm_pow}
\end{figure}

\begin{figure}[H]
  \centering 
  \begin{subfigure}{0.32\textwidth}
    \includegraphics{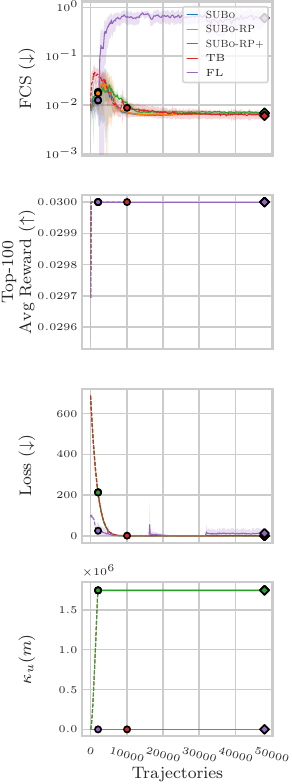}
    \caption{$C=5$}
  \end{subfigure}
    \begin{subfigure}{0.32\textwidth}
    \includegraphics{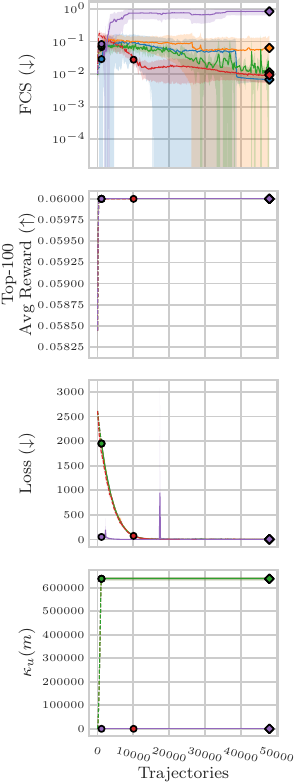}
    \caption{$C=10$}
  \end{subfigure}
    \begin{subfigure}{0.32\textwidth}
    \includegraphics{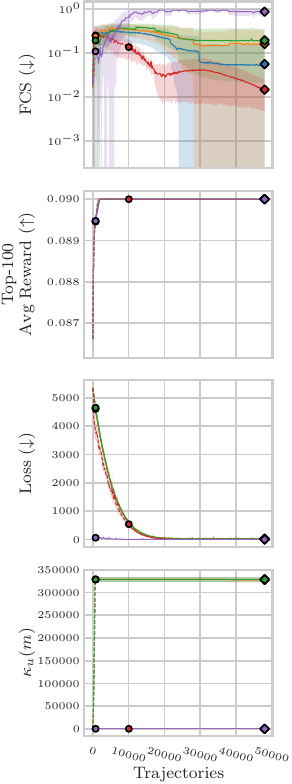}
    \caption{$C=15$}
  \end{subfigure}
  \caption{FCS, Top-100 Average Rewards, Loss, and Coverage for \subo{} and FL/TB-GFNs on CM-REG graphs.} \label{fig:cm_reg}
\end{figure}

\paragraph{Distribution matching} Across all scales and all graphs except CM-REG, vanilla \subo{} consistently reduces the FCS faster than the other methods. \subog{} also performs well in that metric, though at small scale TB-GFN is comparable to it, but it eventually catches up to \subo{}. As expected, \subogp{} is struggling to fit the distribution, likely because it is much more focused on high-reward states than \subo{}(-PR) during offline training. FL-GFNs generally struggle to reduce the FCS. We hypothesize that this is due to the fact that FL-GFN optimizes a detailed-balance objective~\cite{panBetterTrainingGFlowNets2023}, which propagates credit one transition at a time and is known to exhibit higher bias than trajectory balance~\cite{pmlr-v202-madan23a}. As described in Appendix~\ref{app:synth_graphs}, the degree of every vertex is constant in CM-REG graphs. The underlying target distribution is expected to be near uniform. Therefore, it is unsurprising that every approach (except FL-GFN) is comparable in terms of FCS, as the policies are very entropic during early training (they are randomly initialized). Lastly, we notice that in CM-BI graphs, \subog{} successfully reduces the FCS but then through offline training becomes worse at fitting the distribution. This indicates that the softmax sampling of upper bound biases the distribution towards high reward states, as this behavior is barely visible in vanilla \subo{}. This is also likely an important factor explaining why vanilla \subo{} is better at matching the distribution than \subog{}.

\paragraph{Top-100 Average Reward} \subogp{} generally outperforms other methods in that metric. Indeed, in some instances it is significantly better than all other baselines, but at worse it performs comparably to them. This is likely a result of the filtering occurring during even the offline training phase. Indeed, \subog{} also fits this pattern, but it does not detach itself from the other methods as much when $C$ is larger.  At best, it is ranked second when it significantly outperforms other methods, behind \subogp{}. Since the only difference between \subog{} and \subogp{} is this filtering staying, the that likely explains the difference. The ranking of vanilla \subo{} against TB/FL-GFNs depends on the type of graph. In some cases (CM-BI), vanilla \subo{} can outperform TB-GFNs and FL-GFNs, but it is also sometimes worse (ER). In other scenarios, it is comparable. TB-GFNs are better than FL-GFNs, likely because it seems to learn something useful and actually matches the distribution, whereas FL-GFNs seem to be unable to fit it, thus impacting the discovery of high reward states.

\paragraph{Loss} These curves show that all methods successfully decrease their losses, but it is not always synonymous with reducing the FCS, which is the end goal. The loss in FL-GFNs when $C=5$ visibly goes down during the short time it trains online, but then it stops going down. This trend is consistent across scales, but the loss has high variance. We initially believed that the clipping value of the reward was the issue (Appendix~\ref{app:impl_details}), but despite changes in the clipping value, FL-GFNs struggle to learn. The reasons explained above are thus likely candidates for these loss curves. The other approaches all seem to start with a large loss, but eventually successfully learn to reduce it consistently.

\paragraph{Coverage} \subog{} and \subogp{} cover a comparable number of terminating states and \subo{} covers more terminating states than them. This is likely due to the bias induced in the sampling distribution that occurs when sampling upper bounds with a softmax. Indeed, \subog{} and \subogp{} both select their upper bounds in the same manner during online training, and the upper bounds can only be generated during online training (as it requires consuming the budget of queries to the reward function).

\paragraph{Summary}
Overall, \subo{} offers a trade-off between distribution matching and Top-100 Average Reward optimization through its different variants. \subog{} is good on both metrics, but is the best in neither. Vanilla \subo{} is good at matching the distribution, but is not better than \subog{} and \subogp{} on Top-100 Average Reward, being comparable to TB-GFNs and FL-GFNs, on that metric. \subogp{} sacrifices distribution matching in order to sample more high-reward terminating states.

\section{Complexity of \subo{}}
\label{app:complexity}

\begin{algorithm}[ht]
\caption{\subo{} (one online step)}
    \renewcommand{\algorithmiccomment}[1]{\hfill $\triangleright$ #1}
\begin{algorithmic}[1]
\REQUIRE GFN parameters $\theta$, loss $\mathcal L$, policy $\pi$, replay buffer $\mathcal B$, history $\mathcal H$, visited states $\mathcal D$, upper bounds $\mathcal U(\cdot)$, parents $\mathcal{D}_p$, successor sets $\mathrm{Succ}(\cdot)$
\STATE Sample trajectory $\tau$ ending in $x$ with $\pi$
\STATE $\mathcal{H}(s_i(\tau)) \gets  R(s_i(\tau))$ for $i = 0 \dots C$
\STATE $\mathcal B \gets \mathcal B \cup \{(\tau, R(x))\}$
\STATE $\mathcal{U} \gets \mathcal{U}\setminus \{(x,\mathcal{U}(x))\}$
\STATE $p \gets s_{C-1}(\tau)$; $\Delta\mathcal{D} \gets \{s_i(\tau)\}_{i=0}^{C} \setminus \mathcal{D}$;\! $\mathcal{D} \gets \mathcal{D} \cup \Delta\mathcal{D}$
\FOR{$s \in \Delta\mathcal{D}$, $a \in s$}
  \STATE $\mathrm{Succ}(s \setminus \{a\}) \gets \mathrm{Succ}(s \setminus \{a\}) \cup \{a\}$ \COMMENT{Update successor actions of ancestors of $s \in \Delta\mathcal{D}$}
\ENDFOR
\STATE $\Delta\mathcal E \gets \{(s\setminus\{a\}, a) : s \in \Delta\mathcal{D}, a \in s, s\setminus\{a\} \in \mathcal{D}\}$ \COMMENT{Fetch all observed ancestors of $s \in \Delta\mathcal{D}$}
\STATE $\Delta\mathcal E \gets \Delta\mathcal E \cup \{(s, a) : s \in \Delta\mathcal{D}, a \in \mathrm{Succ}(s)\}$  \COMMENT{Fetch all observed successors of $s \in \Delta\mathcal{D}$}
\STATE $\Delta\mathcal U \! \gets\! \{(q, s, a)\! :\! q \in \mathcal{D}_p, (s,a) \in \Delta\mathcal E, s \subsetneq q, a \notin q\}$   \COMMENT{Combine the states with compatible parents into triplets}
\IF{$p \notin \mathcal{D}_p$}
  \STATE $\mathcal{D}_p \gets \mathcal{D}_p \cup \{p\}$
  \STATE $\mathcal{Z} \gets \{s \in \mathcal{D} : s \subsetneq p\}$  \COMMENT{Find subsets of $p$}
  \STATE $\Delta\mathcal U \gets \Delta\mathcal U \cup \{(p, s, a) : s \in \mathcal{Z},\, a \in \mathrm{Succ}(s) \setminus p\}$  \COMMENT{Find pairs of states $s, s\cup \{a\} \in \mathcal{D}$ compatible with $p$}
\ENDIF
\FOR[Compute the new upper bounds and update $\mathcal{U}$]{$(q,s,a) \in \Delta\mathcal U$} 
\IF{$q \cup \{a\} \not \in \mathcal{D}$} 
  \STATE $U \gets \mathcal{H}(s\cup\{a\})\! - \mathcal{H}(s) + \mathcal{H}(q)$ 
  \IF {$q \cup \{a\} \in \mathcal{U}$}
    \STATE $\mathcal{U}(q\cup\{a\}) \gets \min\{\mathcal{U}(q \cup \{a\}, U\}$    
  \ELSE 
      \STATE $\mathcal U(q \cup \{a\}) \gets U$
  \ENDIF
\ENDIF
\ENDFOR
\STATE $\mathcal{T}_\mathcal{B} \leftarrow$ Sample trajectories and rewards from $\mathcal{B}$ 
\STATE $\mathcal{T}_\mathcal{U} \leftarrow$ Sample upper bounds from $\mathcal{U}$ and generate trajectories with $P_{B,\theta}$
\STATE Update parameters $\theta$ with loss $\mathcal L$ on $\mathcal{T}_\mathcal{U} \cup \mathcal{T}_{\mathcal{B}}$. 
\end{algorithmic}
\end{algorithm}

Recalling Algorithm~\ref{alg:subo} above and that $m$ is the number of collected trajectories, we now analyze the computational complexity of this algorithm. We assume the following complexities for operations:
\begin{itemize}
    \item Lookup operations, $a \in A$, for some element $a$ and some set $A$, are $O(1)$ (e.g. via hash maps)
    \item Lookup of incremental states from previous observed states is $O(1)$ (this can be implemented via Zobrist hashing~\citep{zobrist1990new}).
    \item $A\subsetneq B$ operations are $O(|A|)$ for some sets $A,B$
    \item A query to the reward function is in constant time $O(1)$
    \item A query to the policies $\pi$ and $P_B$ are also in constant time ($O(C)$ per trajectory, $O(1)$ per action selection)
    \item Sampling from a set at random is $O(n)$, where $n$ is the number of samples. We will assume $n\ll m$, where $m$ is the number of trajectories, is constant.
    \item Training the GFN on a set of $n$ trajectories is $O(nC)$, as most objectives are a sum of $C$ terms for each trajectory (or a sum of a number of transitions, which can be expressed similarly).
\end{itemize}

Finally, we assume $\pi$ either generates a full trajectory in one shot, thus being in constant $O(1)$ time, or sequentially generates a trajectory, as GFNs do, in which case it operates in $O(C)$ time.
\begin{enumerate}
    \item Lines 1 has a complexity in $O(C)$. Either $\pi$ draws an entire trajectory or it samples one sequentially via some policy over actions, as we assume.
    \item Lines 2 has a complexity in $O(C)$.
    \item Lines 3 has a complexity in $O(1)$.
    \item Lines 4 has a complexity in $O(1)$. Set difference updates scale with the size of the removed set.
    \item Lines 5 has a complexity in $O(C)$, because it check if $C+1$ states are in $\mathcal{D}$. This can be implemented via a hash map, hence requiring $C+1$ $O(1)$ operations.
    \item Lines 6-8 has a complexity in $O(C^3)$. A trajectory contains $C + 1$ states (and thus $\Delta D$ contains at most $C + 1$ states), and a state contains at most $C$ elements, but generating $s \setminus\{a\}$ incurs $C$ operations each time because we must iterate over the set to remove the element (as opposed to an update). 
    
    \item Line 9 has a complexity in $O(C^3)$ because $|\Delta \mathcal{D}| = O(C)$, $|s| = O(C)$ and generating $s \setminus \{a\}$ causes an iteration over $C$ elements. Note that $|\Delta\mathcal{E}| = O(C^2)$, because we only iterate over that states and actions to build the set.
    
    \item Line 10 has a complexity in $O(CN)$. Because a state $s$ has at most $N - |s|$ successors, and $\Delta D$ contains at most $C$ states, we get $O(CN)$. At the end of this step, $|\Delta\mathcal{E}| = O(C^2 + CN)$.
    
    \item Line 11 has a complexity in $O(m|\Delta\mathcal{E}|C) = O(m(C^2 + CN)C) = O(mC^2(C+N))$ as it must iterate through $\Delta \mathcal{E}$ that contains at most $CN + C^2$ elements, and must do so for every encountered parent. As there is one parent per trajectory, there are are most $m$ parents contained in $\mathcal{D}_p$, giving $O(mC(C +N))$. Lastly, each check $s \subsetneq q$ incurs $O(C)$ operations, which occur sequentially and thus we obtain $O(mC^2(C+N))$. Note that $|\Delta \mathcal{U}| = O(m|\Delta\mathcal{E}|) = O(mC(C +N))$, as only the iterations matter in the worst case.
    
    \item Line 14 has complexity $O(mC^2)$ because it must iterate through every observed state. In $m$ trajectories, at most $mC$ unique states have been encountered ($C$ states per trajectories). Then, for each of these states, it must check that state $s$ is a subset of $p$, which incurs $O(C)$ operations at most. Note that $|\mathcal{Z}| = O(|\mathcal{D}|) = O(mC)$.
    
    \item Line 15 has complexity $O(|\mathcal{Z}|N) = O(mCN)$ because it must iterate through the entire set $\mathcal{Z}$ (with size $O(mC)$) and each successor action ($O(N)$ actions at most) that is compatible with the parent $p$ ($O(1)$ time to check membership of $a$ to $p$). The final size of $\Delta \mathcal{U}$ is $O(mC(C +N) + mC)$.
    \item Line 17-26 has complexity $O(|\Delta \mathcal{U}|) = O(mC(C +N) + mC) = O(mC(C+N))$, which comes from the loops. Line 18 creates an incremental set $q \cup \{a\}$, and since $q$ is already known, so is its hash. Then using Zobrist hashing, Line 18 is $O(1)$. Line 19 is $O(1)$ because the hash of observed states is known. Lines 20 to 24 are $O(1)$, and the last two lines close the if and the for.
    \item Line 27 depends on the number of samples $n$ we choose from the buffer. The complexity is $O(n)$.
    \item Line 28 is $O(nC)$, as it first samples $n$ items from $\mathcal{U}$ and then queries $P_{\theta,B}$ $C$ times to generate a full backward trajectory.
    \item Line 29 is $O(nC)$.
\end{enumerate}
As such, the computational complexity of \subo{} \textbf{per trajectory} is $O(mC^2(C+N))$ (Line 11), as it is the dominating term for the entire algorithm.

\section{Implementation Details}
\label{app:impl_details}
We provide in Table~\ref{tab:hyperparameters} the parameters used for our experiments. Each batch of trajectories sampled with \subo(-RP(+)){} or TB/FL-GFN are stored in a replay buffer. For \subo{}, one epoch of training is done with a batch of trajectories composed of a mix of trajectories generated from the upper bounds and from the replay buffer storing the sampled trajectories. The batch is composed of $25\%$ from trajectories from the replay buffer and $75\%$ from trajectories sampled from the terminating states for which we have upper bounds available. For training, classical GFNs operate similarly, but all of the trajectories used in training are from the replay buffer instead.


\begin{table}[ht]
\centering
\caption{Hyperparameter Configuration for GFN Graph Experiments}
\label{tab:hyperparameters}
\small
\begin{tabular}{ll p{6cm}}
\toprule
\textbf{Category} & \textbf{Hyperparameter} & \textbf{Value} \\ 
\midrule
\textbf{GFN Training} & Learning Objective & Trajectory Balance~\citep{malkinTrajectoryBalanceImproved2022} \\ && Forward Looking Detailed Balance~\citep{panBetterTrainingGFlowNets2023} (FL-GFNs only)\\
 & Query Budget & 10,000 \\
 & Batch Size (sampling and offline) & 16 \\
 & Learning Rate ($P_F, P_B, F$) & $10^{-4}$ \\
 & Learning Rate ($Z$, TB only) & $10^{-2}$ \\
\midrule
\textbf{Architecture} & Model Type & Graph Isomorphism Network~\citep{xuHowPowerfulAre2018} \\
 & Embedding Dimension & 128 \\
 & Convolutional Layers & 1 \\
\midrule
\textbf{FCS} & Forward Trajectories & 128 \\
 & Backward Trajectories & 8 \\ 
 & Repeats & 25  \\
\bottomrule
\end{tabular}
\end{table}
\paragraph{Reward} For FL-GFNs, the forward-looking reparameterization relies on intermediate energies $-\log R(s)$, which is undefined at $s_0 = \emptyset$ since $R(\emptyset) = 0$. To avoid errors, we clip from below the reward to a small value. For \subo{}, the clipping value is $10^{-40}$. Technically, this clipping could break submodularity, but $10^{-40} \ll \frac{1}{N}$ which is the smallest reward for one vertex that covers another, so the effect is negligible. For training stability reasons, we chose $10^{-3}$ for FL-GFNs, as the smaller clipping value lead to large, noisy losses.

\paragraph{Code details} All of the code related to GFNs uses torchgfn~\citep{vivianoTorchgfnPyTorchGFlowNet2026} (Apache v2.0). We also use PyTorch, torch-geometric (MIT), scipy (BSD-3) and numpy general computation, as well as OmegaConf (BSD-3) for configuration management.

\paragraph{Hardware Details} We run all of our experiments on H100 NVIDIA GPUs. The CPUs for these experiments vary, as multiple clusters were used. CPUs were either AMD EPYC 9454, Intel Xeon Gold 6448Y, Intel Xeon Platinum 8570.

\end{document}